\documentclass[11pt]{article}

\usepackage{fullpage,times,url,natbib}

\usepackage{amsthm,amsfonts,amsmath,amssymb,epsfig,color,float,graphicx,verbatim}
\usepackage{algorithm,algorithmic}

\usepackage{enumitem}
\usepackage{stackengine}

\usepackage{hyperref}
\hypersetup{
	colorlinks   = true, 
	urlcolor     = blue, 
	linkcolor    = blue, 
	citecolor   = blue 
}

\newtheorem{theorem}{Theorem}

\newtheorem{lemma}{Lemma}

\newcommand{\reals}{\mathbb{R}}
\newcommand{\E}{\mathbb{E}}

\newcommand{\sign}{\mathrm{sign}}

\newcommand{\bw}{\mathbf{w}}

\newcommand{\Ocal}{\mathcal{O}}

\newcommand{\Ncal}{\mathcal{N}}

\newcommand{\norm}[1]{\left\|#1\right\|}

\newcommand{\secref}[1]{Sec.~\ref{#1}}

\renewcommand{\eqref}[1]{Eq.~(\ref{#1})}
\newcommand{\lemref}[1]{Lemma~\ref{#1}}

\newcommand{\thmref}[1]{Thm.~\ref{#1}}

\usepackage{scalerel,stackengine}
\stackMath
\newcommand\reallywidehat[1]{%
\savestack{\tmpbox}{\stretchto{%
  \scaleto{%
    \scalerel*[\widthof{\ensuremath{#1}}]{\kern-.6pt\bigwedge\kern-.6pt}%
    {\rule[-\textheight/2]{1ex}{\textheight}}
  }{\textheight}%
}{0.5ex}}%
\stackon[1pt]{#1}{\tmpbox}%
}
\title{Hardness of Learning Fixed Parities with Neural Networks}
\author{
Itamar Shoshani
\qquad
Ohad Shamir
\vspace{3pt}
\\
Weizmann Institute of Science \\
\texttt{\{itamar.shoshani,ohad.shamir\}@weizmann.ac.il}  
}
\date{}

\begin{document}

\maketitle

\begin{abstract}
  Learning parity functions is a canonical problem in learning theory, which although computationally tractable, is not amenable to standard learning algorithms such as gradient-based methods. This hardness is usually explained via statistical query lower bounds \citep{kearns1998efficient}. However, these bounds only imply that for any given algorithm, there is some worst-case parity function that will be hard to learn. Thus, they do not explain why \emph{fixed} parities -- say, the full parity function over all coordinates -- are difficult to learn in practice, at least with standard predictors and gradient-based methods \citep{abbe2022non}. In this paper, we address this open problem, by showing that for \emph{any} fixed parity of some minimal size, using it as a target function to train one-hidden-layer ReLU networks with perturbed gradient descent will fail to produce anything meaningful. To establish this, we prove a new result about the decay of the Fourier coefficients of linear threshold (or weighted majority) functions, which may be of independent interest. 
\end{abstract}

\section{Introduction}
\label{sec:introduction}
Learning parities -- namely, functions of the form $p_S(x)=\prod_{j\in S} x_j$ for $x\in \{-1,+1\}^d$ and some $S\subseteq \{1,\ldots,d\}$ -- is one of the most basic and well-studied problems in learning theory (\cite{kearns1994introduction,blum2003noise}). In particular, it serves as an important example of a class of predictors that are computationally difficult to learn for large classes of algorithms, even though they are PAC-learnable with modest sample complexity. Indeed, since the parity function can be represented as a linear function over the GF(2) field, the subset $S$ can be easily derived from $\Ocal(d)$ samples of the form $(x,p_S(x))$ (where $x$ is uniform over $\{-1,+1\}^d$), by using Gaussian elimination to solve an associated set of linear equations over GF(2). However, strong empirical evidence suggests that $p_S$ cannot be learned using more standard general-purpose learning methods, and in particular gradient methods, once the dimension $d$ is even moderately large \citep{shalev2017failures,abbe2020poly,nachum2020symmetry}.

The canonical learning-theoretic approach to explain this difficulty is via the statistical queries (SQ) framework (\cite{kearns1998efficient}, see \cite{reyzin2020statistical} for a survey). This framework considers statistical query algorithms, which is any algorithm that can be modeled as a sequence of queries to an SQ oracle. Specifically, each query to the oracle is a function $\phi$ over examples $(x,y)$, and the oracle response is approximation of $\E[\phi(x,y)]$ (where the expectation is with respect to the underlying distribution of $(x,y)$). For example, by taking $\phi$ to be the gradient of a loss function w.r.t. some $(x,y)$ and some class of predictors, this essentially captures applying gradient methods on the expected loss function. In the  paper introducing the SQ framework, \cite{kearns1998efficient} proved that \emph{any} SQ algorithm will require exponentially (in $d$) many queries to learn some worst-case parity function $p_S$. Very roughly, this is because $\{p_S\}_{S\subset \{1,\ldots,d\}}$ is a set of $2^d$ orthogonal functions, and therefore for any fixed algorithm, one can find subsets $S$ for which the algorithm's oracle queries will provide little information about $S$.

Although this is a foundational result with many applications, it suffers from two weaknesses. The first is that the oracle is not required to return $\E[\phi(x,y)]$ precisely, but rather a (possibly adversarial) perturbation of this value. Although this adversarial perturbation is exponentially small in the case of parities (and can be mitigated in some situations \citep{yang2005new,feldman2017statistical}), it means that the hardness result does not apply as-is to standard gradient descent or stochastic gradient descent, where we really get $\E[\phi(x,y)]$ or some unbiased stochastic estimate of it. A second weakness of the SQ hardness result is that it is worst-case over all parities $f_S$: Namely, it shows that for any algorithm, there exists some (possibly algorithm-dependent) hard parity function $p_S$. This worst-case assumption also appears in other hardness results for parities, such as \cite{raz2018fast}. However, in practice one observes that even if we fix the parity -- say, the full parity function $p_S$ where $S=\{1,\ldots,d\}$ -- standard learning methods fail to learn this function, already for moderately large $d$. 
This hardness of learning \emph{fixed} parities is not explainable by the SQ framework.

The question of why learning fixed parities is hard was pointed out and studied recently in \cite{abbe2022non}. In that paper, the authors proved a hardness result for learning certain fixed target parity functions $p_S$, assuming some invariance properties of the training algorithm, and using symmetry properties of the class of parity functions $p_S$ for fixed subset size $|S|$. However, their results crucially apply only when $|S|$ is bounded far away from $0$ and $d$,say when it equals $d/2$ (so that the number of subsets $S$ with some fixed $|S|$ is exponentially large in $d$). In particular, it says nothing about the full parity function, even though in practice learning the full parity function using gradient methods is not easier than parities $p_S$ with a smaller value of  $|S|$. The authors left this issue as an open problem, which is the focus of our paper. 

To describe our results, we begin by pointing out that any hardness result for learning a fixed target function is necessarily algorithm- and model-dependent. Indeed, any fixed function can be trivially learned by the ``algorithm'' that simply returns this function. Moreover, the recent work by \cite{abbe2020universality, abbe2021power} imply that any parity can actually be learned with stochastic gradient descent, using a sufficiently contrived neural network architecture and initialization strategy (which cause the network to implement generic Turing machines, and Gaussian elimination in particular). Thus, the question must be studied in the context of some ``reasonable'' class of predictors and algorithms. 

Concretely, in this paper, we consider perhaps the simplest class of neural networks that can express parity functions: Namely, one-hidden-layer ReLU neural networks of the form $x\mapsto \sum_{j=1}^{n}u_{j}\left[w_{j}^{T}x+b_{j}\right]_{+}$, where $[z]_+=\max\{z,0\}$ is the ReLU function. Our main results are as follows:
\begin{itemize}
\item We prove that for any \emph{fixed} parity function $f_S$, even though a network as above can express it, running a standard perturbed gradient descent (PGD) algorithm with standard Gaussian initialization will fail at returning anything meaningful: Specifically, exponentially (in $|S|$) many iterations will be required to reduce the expected loss by more than an exponentially small value. For example, this applies to the full parity function (where $S=\{1,\ldots,d\}$), which is not covered by any previous works we are aware of, and explains why it is difficult to learn already for moderately large $d$, at least in this framework. We note that the result applies to gradient descent with stochastic, non-adversarial perturbations (unlike most of the SQ literature, which requires adversarial perturbations). 
\item We prove a similar result for the squared loss, in the case where the network is composed of a single ReLU neuron $x\mapsto [w^\top x+b]_+$. Specifically, we show that even though there exists a ReLU neuron $N_{\theta}$ such that $\E_{x}[(N_{\theta}(x)- p_S(x))^2] \leq 1-\text{poly}(1/|S|)$, PGD will only achieve values exponentially close to $1$, unless the number of iterations is exponential in $|S|$. 
\item In order to prove our results, we develop a key result on the Fourier coefficients of linear threshold (or weighted majority) functions, which may be of independent interest. Specifically, consider the function $f_S(w) = \E_{x\in \{-1,+1\}^d}[p_S(x)\mathbf{1}_{\{w^\top x>0}\}]$, where $\mathbf{1}_{\{\}}$ is the indicator function. In Boolean analysis, this corresponds to the Fourier coefficient (indexed by $S$) of the linear threshold / weighted majority function $x\mapsto \mathbf{1}_{\{w^\top x>0\}}$. These coefficients are very well-studied and have many applications (see the book \cite{odonnell2021analysisbooleanfunctions} for a survey). However, as far as we know, the behavior of these coefficients is only well-understood for very small constant $|S|$, or specifically for the plain majority function (where $w$ is the all-ones vector). Indeed, for the majority function, it is well-known that the Fourier spectrum decays polynomially in $|S|$, with the Fourier coefficient corresponding to $S$ being as large as $O(|S|^{-1/2})$ \cite[Eq. 5.11 and Theorem 5.19]{odonnell2021analysisbooleanfunctions}. However, we show that for more general weighted majority functions, the coefficients are \emph{exponentially} small in $|S|$, in an average-case sense: Namely, if we sample $w$ uniformly from a spherically symmetric distribution, then $|f_S(w)|\leq \exp(-\Omega(|S|))$ with overwhelming probability. This result is key to our other theorems, as we relate the size of $f_S(w)$ to the magnitude of the gradients encountered by the gradient descent algorithm during training.
\end{itemize}

We note that \cite{abbe2022non} proved that PGD on certain one-hidden-layer neural networks can learn parity functions in a weak sense (namely, that the objective function can be reduced by an inverse polynomial value). This does not contradict our hardness results, as they consider networks with a specific bump-like activation function, and a specific (integer-valued) initialization for the weights. In contrast, we focus on standard ReLU neural networks and standard (Gaussian) initializations. 

Our paper is structured as follows. In \secref{sec:preliminaries}, we present some notation and formal definitions. In \secref{sec:layer}, we present our results for one-hidden-layer neural networks, as well as the result on the Fourier coefficients of linear threshold / weighted majority functions. In \secref{sec:single}, we present our results for a single neuron network and the squared loss. We conclude in \secref{sec:conclusion} and discuss open questions that remain. Full proofs of our results appear in the appendices.

\section{Preliminaries}\label{sec:preliminaries}

\paragraph{Notation.}
Given some $S\subseteq [d]:=\{1,2,\ldots,d\}$, define the parity function on $S$ as $p_S(x)=\prod_{i\in S}x_i$. 
For any $S\subseteq[d]$ we denote the corresponding indicator vector in $\{0,1\}^d$ as $\mathbf{1_S}:=(\mathbf{1}_{\{i\in S\}})_{i=1}^d$. The expectation $\E_{x\in\{\pm 1\}^d}$ is over $x$ uniformly distributed over $\{-1,+1\}^d$. The distribution $\Ncal(0,\sigma^2\cdot I)$ refers to a zero-mean Gaussian multivariate distribution with covariance matrix being the identity matrix scaled by $\sigma^2$. $I_d$ refers specifically to the $d\times d$ identity matrix. Finally, for some boolean function $g:\{+1,-1\}^d\rightarrow\reals$ and subset $S\subseteq[d]$, the Fourier coefficient of $g$ corresponding to $S$ is defined as 
\[
\hat{g}(S)~=~\E_{x\in\{\pm 1\}^d}\left[g(x)p_S(x)\right]~=~\E_{x\in\{\pm 1\}^d}\left[g(x)\prod _{i\in S}x_i \right]~.
\]

\paragraph{Perturbed Gradient Descent.} 
In this paper, we focus on studying the performance of perturbed gradient descent (PGD), a simple and representative gradient optimization method for which our proof techniques are most readily applicable (see for example \citet{jin2017escape,du2017gradient,zhang2021escape}). This algorithm appears in the literature under several additional names, such as noisy gradient descent (e.g., \citet{abbe2022non}), and gradient Langevin dynamics (e.g., \citet{xu2018global,dalalyan2017further,durmus2017nonasymptotic}). We will focus here on Gaussian perturbations, under which the algorithm is defined as follows: Given some differentiable objective function $F(\theta)$ (where $\theta$ is a vector in some Euclidean space), a variance parameter $\sigma^2>0$, and a step size parameter $\eta>0$, we choose an initial point $\theta_0$ by sampling $\Ncal(0,\sigma^2\cdot I)$, and iteratively compute
\[
\theta_{t+1}=\theta_t-\eta \nabla F_S(\theta)-\xi_t,
\]
where $\xi_1,\xi_2,\ldots$ are independent random variables distributed as $\Ncal(\mathbf{0},\sigma^2 \cdot I)$. We let $T$ be the total number of iterations performed by the algorithm, with $\theta_T$ being the resulting output. 

When applying PGD for objective functions involving the ReLU function $z\mapsto[z]_+$ (which is formally not differentiable at $0$), we will use the convention that its derivative for all $z$ is $\mathbf{1}_{\{z>0\}}$ -- in particular, that it is $0$ at $z=0$. This is without much loss of generality, and our results can be easily modified to handle other conventions.

\section{One-Hidden-Layer ReLU Neural Networks}\label{sec:layer}

In this section, we consider the problem of learning a parity function using a one-hidden-layer ReLU neural networks. This network takes the form
\[
N_{\theta}(x)=\sum_{j=1}^{n}u_{j}\left[w_{j}^{T}x+b_{j}\right]_{+}~~~\text{where}~~~\theta=\left(u_{j},w_{j},b_{j}\right)_{j=1}^{n}~,
\]
where $n$ is a width parameter, $u_j,b_j\in \reals$,  $w_j\in \reals^d$, and $[z]_+=\max\{0,z\}$ is the ReLU function. Moreover, we will consider the linear (or correlation) loss and a uniform distribution over $x\in\{-1,+1\}^d$, in which case our objective function can be written as 
\[
F_S(\theta):= - \E_{x\in\{\pm 1\}^d}\left[N_{\theta}(x)p_S(x)\right]~=~
- \E_{x\in\{\pm 1\}^d}\left[\prod_{i\in S}x_i\sum_{j=1}^{n}u_{j}\left[w_{j}^{T}x+b_{j}\right]_{+}\right]~.
\]
Note that this function can trivially takes a value of $0$ (say when $\theta=0$). On the other hand, if $N_{\theta}(x)=p_S(x)$ for all $x$, then $F_S(\theta)=-1$. Thus, for any meaningful performance, it is necessary that the network parameters $\theta$ are such that $F_S(\theta)$ will be significantly smaller than $0$. Our goal will be to show that this is in fact not achievable, at least using PGD.

Before continuing, it is necessary to establish that a one-hidden-layer ReLU network, with modestly-sized weights, is sufficient to express parity functions (otherwise, we might trivially fail,  simply because no choice of $\theta$ leads to a reasonable approximation of the parity function). This is formalized in the following folklore result: 
\begin{theorem}
\label{thm:weights}
    For any $S\subseteq[d]$, there exists a width $n=|S|+1$ network $N_{\theta}(x)$ as above, such that $\norm{\theta}\leq 6|S|^\frac{3}{2}$ and $N_{\theta}(x)=p_S(x)$ for all $x\in \{\pm 1\}^d$. Thus, for this $\theta$, $F_S(\theta) = -1$.
\end{theorem}
The key observation behind the proof is that $p_S(x)$ is a univariate function of the inner product $\mathbf{1}_S^\top x$, which takes values in a discrete finite set. Therefore, using a weighted linear combinations of ReLU neurons of the form $[\mathbf{1}_S^\top x+b]_+$, we can express the function $p_S$ precisely. 

\subsection{Main Result}

We now turn to present our main hardness result. Simply put, it states that unless some algorithm/architecture parameter (such as the number of iterations) is exponentially large in $|S|$, then PGD will leave $F_S(\theta)$ exponentially close to $0$.

\begin{theorem}\label{thm:hard}
    Suppose that $d\geq 30$ and $|S|\geq 72\cdot\max \{\ln(6nd\sigma),\ln(5nd\sigma^2(T+1))\}$. Then after $T$ steps of PGD, it holds with probability at least $1-(\frac{\eta\sqrt{T}}{2\sigma}+T+2)\exp(-|S|/18)$ that
    \[
    \left|F_S(\theta_T)\right|~<~ \exp(-|S|/18)~.   
    \]
\end{theorem}
Again, we point out that due to \thmref{thm:weights}, this failure is not due to lack of expressiveness, but rather the 
failure of the algorithm's dynamics to reach a meaningful solution. Thus, unlike SQ hardness results which are worst-case over all parities, the theorem implies that learning any \emph{specific} parity function $p_S$ will be hard in our setting, with the hardness increasing exponentially in $|S|$. In particular, the number of required iterations for any meaningful guarantee is super-polynomial in the input dimension $d$, as long as $|S|$ is more than logarithmic in $d$.

The proof of the theorem relies on a key result, which we believe to be of independent interest. To motivate this result, let us consider the derivative of $F_S(\theta)$ with respect to some individual weight vector $w_j$. It is easily verified that this equals
\begin{equation}\label{eq:Fderiv}
\frac{\partial}{\partial w_j} F_S(\theta)~=~-\E_{x\in\{\pm 1\}^d}\left[p_S(x)\cdot u_j\mathbf{1}_{w_j^\top x+b_j>0}x\right]~.
\end{equation}
The proof of \thmref{thm:hard} ultimately relies on arguing that such derivatives tend to be exponentially small (in $|S|$) for ``most'' choices of $\theta$, and thus the iterates of PGD are nearly equivalent to a plain Gaussian random walk, which does not lead to any meaningful performance. We will return later to the other components of the proof, but for now let us focus on the question of why expressions such as in \eqref{eq:Fderiv} tend to be small.

For simplicity of the exposition, let us drop the minus sign, $u_j$ and the $x$ terms in \eqref{eq:Fderiv}. Then the expression becomes 
\begin{equation}\label{def:f}
    f_S(w,b):=\E_{x\sim\{\pm1\}^d}\left[p_S(x)\mathbf{1}_{\{w^\top x+b>0\}}\right]~.
\end{equation}
This is nothing more than the Fourier coefficient (corresponding to $S$) of the linear threshold/weighted majority function\footnote{This function is often defined in the literature as $x\mapsto \sign(w^\top x+b)$. However, since $\mathbf{1}_{\{w^\top x+b>0\}}=\frac{1}{2}(\sign(w^\top x+b)+1)$, 
the functions share the same Fourier coefficients for all $S\neq \emptyset$, up to a multiplicative factor of $\frac{1}{2}$.} $x\mapsto \mathbf{1}_{\{w^\top x + b>0\}}$. This function plays a prominent role in Boolean analysis and its applications \citep{odonnell2021analysisbooleanfunctions}, and there has been much work on its Fourier coefficients. However, perhaps surprisingly,  the behavior of these coefficients is apparently well-understood only for very small constant $|S|$, or specifically for the plain majority function (where $w$ is the all-ones vector\footnote{Or slightly more generally, any $w\in \{-\alpha,\alpha\}^d$ where $\alpha\neq 0$, since by symmetry and scale invariance $f_S((w,0))$ is the same for all such $w$.} and $b=0$). For the majority function, its Fourier spectrum is well-known to decay polynomially in $|S|$, with some coefficients being as large as  $O(|S|^{-1/2})$ (say when $|S|=d$, see \citet[Eq. 5.11 and Theorem 5.19]{odonnell2021analysisbooleanfunctions}). However, as far as we know, the decay of the Fourier coefficients for more general linear threshold functions has not been addressed in previous work. 

In the following theorem, we resolve this question, and prove that in sharp contrast to the majority function, where the Fourier coefficients can be polynomially large in $|S|$, the Fourier coefficients of weighted majority are \emph{exponentially} small in $|S|$, in an average-case sense:

\begin{theorem}\label{thm:f0bound}
Suppose that $d\geq 2$, and fix some $S\subseteq [d]$ such that $|S|\geq 2$. Then the Fourier coefficient (corresponding to $S$) of the linear threshold function $x\mapsto \mathbf{1}_{w^\top x>0}$, namely $f_S(w,0)$, satisfies
\begin{equation*}
\mathbb{E}_{w\sim\mathcal{N}\left(0,I\right)}\left[f_S^{2}(w,0)\right]~<~6\exp\left(-\frac{|S|}{4}\right)~.
\end{equation*}
\end{theorem}
By Markov's inequality, this implies that the Fourier coefficient corresponding $S$ is exponentially small in $S$, with probability exponentially close to $1$. Moreover, since $f_S(w,0)$ is invariant to scaling of $w$, the same result will hold for any spherically symmetric distribution of $w$. We note that a similar result also holds for $f_S(w,b)$ for $b\neq 0$ (see \thmref{thm:fbound_with_b} in the appendix), and this more refined result is the one we actually utilize in the proof of \thmref{thm:hard}. 

Connecting this theorem with \eqref{def:f}, we see that the gradient of our objective function (at least with respect to the $w_j$ parameters) is exponentially small in $|S|$, for nearly all $w$ (with respect to a spherically symmetric random choice of $w$). 
Thus, we would expect that applying PGD with such exceedingly small gradients will tend to resemble a plain Gaussian random walk, which is unhelpful for optimizing $F_S$. Unfortunately, actually proving this is not that straightforward, since the theorem does not provide any information where exactly is the gradient small in the parameter space. Indeed, we know that there are values of $w$ for which the gradient will be quite big: For example, when $w=\mathbf{1}_S$, we know that $f_S(w,0)$ can be polynomial in  $|S|^{-1}$. We therefore need to restrict ourselves to situations where the marginal distribution of the parameters tend to remain spherically symmetric, as required to apply the theorem. This is the technical reason why our results in this section focus on the linear loss and on PGD with zero-mean Gaussian perturbations.

Despite this, it is not too difficult to extend \thmref{thm:hard} in various directions. For example, instead of Gaussian perturbations in the PGD iterations, one can potentially utilize other spherically symmetric perturbations (under appropriate technical assumptions). Also, instead of considering a constant step size, one can easily adapt the proof to varying step sizes $\{\eta_t\}_{t=1}^T$ (as long as their magnitude remains reasonably bounded). It is also possible to trade-off between the bound on the probability stated in the theorem, and the magnitude of the bound on $|F_S(\theta_T)|$.

We will now turn to sketch the proof of \thmref{thm:f0bound} (with the full proof appearing in the appendix), followed by the proof of \thmref{thm:hard}. 

\subsection{Proof Sketch of \thmref{thm:f0bound}}

The proof is based on the following simple trick: We begin by noting that
\begin{align*}
\E_{w\sim\mathcal{N}\left(0,I\right)}[f_S^2(w,0)]~&=~\E_{w\sim\mathcal{N}\left(0,I\right)}\left[\E_{x\sim\{\pm1\}^d}^2\left[p_S(x)1_{\{w^\top x>0\}}\right]\right]\\
&=~\E_{w\sim\mathcal{N}\left(0,I\right)}\left[\E_{x\sim\{\pm1\}^d}\left[p_S(x)1_{\{w^\top x>0\}}\right]\cdot\E_{y\sim\{\pm1\}^d}\left[p_S(y)1_{\{w^\top y>0\}}\right]\right]~,    
\end{align*}
where $y$ is an independent copy of the random variable $x$. Rearranging the expectations, this equals
\begin{align*}
\E_{x,y,w}\left[p_S(x)p_S(y)1_{\{w^\top x>0\}}1_{\{w^\top y>0\}}\right]~&=~
\E_{x,y}\left[p_S(x)p_S(y)\E_w\left[1_{\{w^\top x>0~\wedge w^\top y>0\}}\right]\right]\\
&=~
\E_{x,y}\left[p_S(x)p_S(y)\Pr\left(w^\top x>0~\wedge w^\top y>0~|~x,y\right)\right]~.
\end{align*}
Recalling that $w$ is drawn from a spherically symmetric distribution, it follows that the probability in the expression above is simply the measure of the intersection of two hemispheres (centered in the direction of $x$ and $y$ respectively) on the unit $d$-dimensional sphere. By standard calculations, this equals $\frac{1}{2\pi}\left(\pi-\arccos\left(\frac{x^\top y}{d}\right)\right)$. Moreover, since $\arccos$ is an analytic function, we can write this expression using its Taylor series $\sum_{j=0}^\infty\alpha_j\left(\frac{x^\top y}{d}\right)^j$ for suitable coefficients $\alpha_j$. Plugging this back into the displayed equation above, we get
\[
\E_{x,y}\left[p_S(x)p_S(y)\sum_{j=0}^\infty\alpha_j\left(\frac{x^\top y}{d}\right)^j\right]
~=~\sum_{j=0}^{\infty}\alpha_j \E_{x,y}\left[\prod_{i\in S}x_iy_i\cdot \left(\frac{x^\top y}{d}\right)^j\right]~.
\]
Since $x,y$ are uniformly distributed on $\{-1,+1\}^d$, and $(x^\top y)^j$ is a polynomial of degree $2j$, it is easy seen that the first $|S|$ terms in the sum above vanish. The higher order terms can be shown to be exponentially small in $|S|$ and rapidly decaying, eventually leading to the theorem statement. 

\subsection{Proof of \thmref{thm:hard}}

Fixing the objective function $F_S$, recall that $T$ PGD iterates $\{\theta_t\}_{t=0}^{T}$ are defined as $\theta_0\sim\Ncal(0,\sigma^2\cdot I)$, and 
\begin{equation}\label{def:delta0}
\theta_{t+1}=\theta_t-
\Delta_t~~,~~
\text{where}~~
\Delta_t=\eta \nabla F_S(\theta_t)+\xi_t~~,~~ \xi_t\sim\Ncal(0,\sigma^2\cdot I)~.
\end{equation}
This implies that $\{\theta_t\}_{t=0}^T$ is a sequence of random variables induced by the random variables $\{\xi_t\}$. We will let $P$ denote the probability law of these random variables.

As discussed previously, \thmref{thm:f0bound} (and its variant \thmref{thm:fbound_with_b} in the appendix) can be shown to imply that $\nabla F_S(\theta)$ is small nearly everywhere, which intuitively (in light of \eqref{def:delta0}) implies that PGD will fail in returning a meaningful result. To formalize this and prove \thmref{thm:hard}, the intuitive approach is to use information-theoretic tools, to argue that the small gradients will be dominated by the Gaussian noise injected at each PGD iteration, and thus the distribution of the PGD iterates will be nearly the same as if the gradients were exactly zero and the iterates are just the result of a plain Gaussian random walk (that is, $\theta_{t+1}=\theta_t-\xi_t$ for all $t$). However, a technical obstacle to this line of argument is that the gradients are not actually small everywhere, only in some large portion of the parameter space. Thus, it is not a-priori clear how to rule out the possibility that the PGD dynamics lead us to areas of the parameter space where the gradients are large. 

To overcome this obstacle, we  define an auxiliary probability law $\tilde{P}$ over $\{\theta_t\}_{t=0}^T$(with respect to the same set of random variables $\{\xi_t\}$) as followed: First, fix 
\[
\varepsilon:=\exp\left(-|S|/ 18\right)~,
\]
and define the vector operator
\begin{align*}
    [z]_\varepsilon=\begin{cases}
        0 & \norm{z}\leq\varepsilon \\
        z & \text{otherwise}
    \end{cases}~.
\end{align*}
The probability law $\tilde{P}$ is the one induced by the alternative update rule
\begin{equation}\label{def:delta1}
\theta_{t+1}=\theta_t-
\tilde{\Delta}_t~~,~~
\text{where}~~
\tilde{\Delta}_t=\eta\left[\nabla F_S(\theta_t)\right]_\varepsilon 
    +\xi_t~~,~~ \xi_t\sim\Ncal(0,\sigma^2\cdot I)~,
\end{equation}
and where $\theta_0\sim\Ncal(0,\sigma^2\cdot I)$ as before. In words, whereas $P$ refers to the actual distribution induced by $T$ PGD iterations (as defined above), $\tilde{P}$ refers to the distribution induced by $T$ altered PGD iterations, where the gradient is replaced by $0$ if its norm is less than $\varepsilon$ (in which case, the update $\tilde{\Delta}_t$ is just a random Gaussian step).  Note that under both $P$ and $\tilde{P}$, $\{\theta_t\}_{t=0}^{T}$ are defined as functions of the same set of random variables $\{\xi_t\}_{t=1}^T$ and $\theta_0$, hence they share the same sample space.

With this auxiliary $\tilde{P}$ in hand, our proof strategy is as follows: We first show (\lemref{lem:TV01}) that $P,\tilde{P}$ are nearly the same in terms of total variation distance -- intuitively, this is because for \emph{any} realization of $\theta_t$, the distribution of $\theta_{t+1}-\theta_t$ conditioned on $\theta_t$ is a Gaussian with near-zero mean, under both $P$ and $\tilde{P}$. Thus, for probabilistic guarantees on $F_S(\theta_T)$ under $P$, it suffices to study the random variable $F_S(\theta_T)$ under $\tilde{P}$. Crucially, whereas under $P$ the PGD updates $\Delta_t$ are generally not equal to $\xi_t$ (due to the $\eta\nabla F_S(\theta_t)$ terms), under $\tilde{P}$ they are, except for some very small subset $E$ of the parameter space where the gradients are large. This greatly simplifies the analysis, since now we can argue that under $\tilde{P}$ the iterates $\{\theta_t\}_{t=0}^{T}$ follow a plain Gaussian random walk, \emph{until} they hit the set $E$. 
However, the probability of a plain Gaussian random walk to hit the set $E$ in a bounded number of iterations is small, since the measure of $E$ is small. Thus, with high probability, the $T$  iterates will not hit the set $E$, and therefore their distribution will be nearly the same as a plain Gaussian random walk. Since such a random walk is unlikely to return meaningful results (as we show in \lemref{lem:randomF_small}), the theorem follows.

To formalize this proof strategy, we begin by stating an upper bound on the total variation distance between $P$ and $\tilde{P}$:

\begin{lemma}\label{lem:TV01}
    For $P,\tilde{P}$ as defined in \eqref{def:delta0} and \eqref{def:delta1}, it holds that
    \begin{equation*}
    TV\left(P,\tilde{P}\right)~\leq~ \frac{\varepsilon\eta\sqrt{T}}{2\sigma}~,    
    \end{equation*}
where $TV$ is the total variation distance (Namely, $TV(P,\tilde{P}):=\sup_{A\in\mathcal{F}}|P(A)-Q(A)|$ where $\mathcal{F}$ is the relevant $\sigma$-algebra).
\end{lemma}
The proof of this lemma appears in the appendix. With this lemma in hand, and by the definition of total variation distance, we can start to upper bound $P(|F_S(\theta_T)|\geq\varepsilon)$ as follows: 
\begin{align}
    P\left(\left|F_S(\theta_T)\right|\geq \varepsilon \right) &\leq TV(P,\tilde{P})+ \tilde{P}\left(\left|F_S(\theta_T)\right|\geq \varepsilon \right)\nonumber\\
    &\label{eq:boundF_1}\leq \frac{\varepsilon\eta\sqrt{T}}{2\sigma}+ \tilde{P}\left(\left|F_S(\theta_T)\right|\geq \varepsilon \right)~.
\end{align}

Next comes the second part of the proof, where we show that the distribution of the PGD iterates under $\tilde{P}$ is similar to a plain Gaussian random walk.
To do so, we first present the following important lemma, which implies that the gradient of $F_S$ is exceedingly small nearly everywhere, with respect to a scaled standard Gaussian measure: 
\begin{lemma}\label{lem:gradsmall}
Suppose that $d>30, \, |S|\geq 72\ln(6nd\sigma) $ and that $\theta\sim\Ncal(0,\sigma^2 I)$. Then 
\begin{align*}
    &\E_{\theta}\left[\norm{\nabla F_S(\theta)}\right]\leq\exp(-|S|/9)\\
    &\E_{\theta}\left|\partial_{u_j}F_S\right|= \left| \E_{x\in\{\pm1\}^d}\left[p_S(x)[w^\top x+b]_+\right]\right| \leq5d\sigma\exp(-|S|/8) 
\end{align*}
and therefore by Markov's inequality,
\[
\Pr_{\theta}\left(\norm{\nabla F_S(\theta)}\geq \exp(-|S|/18)\right)~\leq~ \exp(-|S|/18)~.
\]
\end{lemma}
This lemma's proof relies mainly on \thmref{thm:fbound_with_b} and on the similarity between $F_S$ and $f_S$ (see \eqref{def:f}), and its full proof can be found in the appendix. Equipped with this lemma, we can go ahead and deal with $\tilde{P}\left(\left|F_S(\theta_T)\right|\geq \varepsilon \right)$.
For simplicity, we will use the notation $\theta_0=\xi_{-1}$ (it will allow us to rewrite the expression $\theta_0+\sum_{t=0}^j\xi_t$ as $\sum_{t=-1}^j\xi_t$). We have:
\begin{align}
    \tilde{P}\left(\left|F_S(\theta_T)\right|\geq \varepsilon\right)
    &=\tilde{P}\left(\left|F_S(\theta_T)\right|\geq \varepsilon\wedge\exists j: \, \norm{\nabla F_S\left(\sum_{t=-1}^j\xi_t\right)}\geq\varepsilon\right)\nonumber\\
    &\quad\quad\quad+\tilde{P}\left(\left|F_S(\theta_T)\right|\geq \varepsilon\wedge\forall j: \, \norm{\nabla F_S\left(\sum_{t=-1}^j\xi_t\right)}<\varepsilon\right)\nonumber\\
    &\leq \tilde{P}\left(\exists j: \, \norm{\nabla F_S\left(\sum_{t=-1}^j\xi_t\right)}\geq\varepsilon\right)\nonumber\\
    &\quad\quad\quad+\tilde{P}\left(\left|F_S(\theta_T)\right|\geq \varepsilon\wedge\forall j: \, \norm{\nabla F_S\left(\sum_{t=-1}^j\xi_t\right)}<\varepsilon\right)~,\label{eq:P'_bound_1}
\end{align}
where the first equality is through splitting the probability into an event and its complement, and the inequality is due to inclusion. 
Observe that under $\tilde{P}$, which we defined by setting $\theta_{t+1}=\theta_t-\eta[\nabla F_S(\theta_t)]_\varepsilon+\xi_t$, the event $\{\forall j: \,\, \norm{\nabla F_S\left(\sum_{t=-1}^j\xi_t\right)}<\varepsilon\}$ implies the event $\{\forall j>0: \,\, \theta_{j+1}=\theta_j+\xi_t \}=\{\forall j: \,\, \theta_{j+1}=\sum_{t=-1}^j\xi_t\}$ (as can be easily seen via induction). Using this, we can bound the second term in \eqref{eq:P'_bound_1} as follows:
\begin{align*}
     \tilde{P}\left(\left|F_S(\theta_T)\right|\geq \varepsilon\wedge\forall j: \, \norm{\nabla F_S\left(\sum_{t=-1}^j\xi_t\right)}<\varepsilon\right) &\leq \tilde{P}\left(\left|F_S(\theta_T)\right|\geq \varepsilon\wedge\forall j: \,\, \theta_{j+1}=\sum_{t=-1}^j\xi_t\right)\\
     = \tilde{P}\left(\left|F_S\left(\sum_{t=-1}^{T-1}\xi_t\right)\right|\geq\varepsilon\wedge\forall j: \,\, \theta_{j+1}=\sum_{t=-1}^j\xi_t\right)&\leq \tilde{P}\left(\left|F_S
     \left(\sum_{t=-1}^{T-1}\xi_t\right)\right|\geq\varepsilon\right)~.
 \end{align*}

Using this inequality and then the union bound, we can upper bound \eqref{eq:P'_bound_1} as follows:
\begin{align*}
    &\tilde{P}\left(\exists j \, \norm{\nabla F_S\left(\sum_{t=-1}^j\xi_t\right)}\geq\varepsilon\right)+\tilde{P}\left(F_S\left(\sum_{t=-1}^{T-1}\xi_t\right)\geq\varepsilon\right)\\
    &\leq \sum_{j=-1}^{T-1} \tilde{P}\left(\norm{\nabla F_S\left(\sum_{t=-1}^j\xi_t\right)}\geq\varepsilon\right)+\tilde{P}\left(F_S\left(\sum_{t=-1}^{T-1}\xi_t\right)\geq\varepsilon\right)\\
    &\leq (T+1)\varepsilon + \tilde{P}\left(F_S\left(\sum_{t=-1}^{T-1}\xi_t\right)\geq\varepsilon\right)~,
\end{align*}
where the last transition was due to \lemref{lem:gradsmall} and the fact that $\sum_{t=-1}^j\xi_t\sim\mathcal{N}(0,(j+2)\sigma^2I)$. Plugging the above and \eqref{eq:P'_bound_1} into \eqref{eq:boundF_1}, we get:

\begin{align}
    \label{eq:Fbound_2} P\left(\left|F_S(\theta_T)\right|\geq \varepsilon\right)
    \leq 
    \frac{\eta\varepsilon\sqrt{T}}{2\sigma}+(T+1)\varepsilon + \tilde{P}\left(\left|F_S\left(\sum_{t=-1}^{T-1}\xi_t\right)\right|\geq\varepsilon\right)~.
\end{align}

Getting to the third and final part of the proof, we need to bound  $\tilde{P}\left(\left|F_S\left(\sum_{t=-1}^{T-1}\xi_t\right)\right|\geq\varepsilon\right)$, namely the probability that a plain Gaussian random walk will make $F_S$ more than exponentially close to $0$. For that, we present the following lemma (whose proof is in the appendix):
\begin{lemma}
\label{lem:randomF_small}
    For $d>30$,  $|S|\geq\max\{72\ln\left(8nd\sigma\sqrt{T+1}\right),72\ln\left(5nd\sigma^2(T+1)\right)\}$, and $\theta\sim\Ncal(0,(T+1)\sigma^2 I)$, it holds that
    \begin{equation*}
    \text{Pr}\left(\left|F_S(\theta)\right|\geq  \varepsilon\right)\leq \varepsilon~,
    \end{equation*}
    where we recall that $\varepsilon=\exp(-|S|/18)$.
\end{lemma}

Since $\sum_{t=-1}^{T-1}\xi_t\sim\mathcal{N}(0,(T+1)\sigma^2I)$ we may use \lemref{lem:randomF_small} on \eqref{eq:Fbound_2} and finish our proof:
\begin{align*}
  P\left(\left|F_S(\theta_T)\right|\geq \varepsilon\right)
    &\leq 
    \frac{\eta\varepsilon\sqrt{T}}{2\sigma}+(T+1)\varepsilon + \varepsilon.
\end{align*}

\section{Single Neurons} \label{sec:single}

As discussed earlier, the hardness result in the previous section is for the expected linear loss. Although a simple and convenient loss, one may wonder whether our result can be extended to other more common losses, such as the squared loss -- in which case our objective function takes the form
\[
F_S(\theta)~=~\E_{x\sim\{\pm1\}^d} \left[\left(N_{\theta,S}(x)-p_S(x)\right)^2\right]~.
\]
However, with the squared loss, the quadratic terms complicate the optimization dynamics, and the marginal distribution of the PGD iterates are no longer guaranteed to remain (nearly) spherically symmetric, as required for our proof techniques. Nevertheless, we will show in this section that a similar result can be shown in the special case of a \emph{single-neuron} ReLU neural network, that is
\[
N_{\theta,S}(x)~=~[w^\top x+b]_+~~\text{or}~~ N_{\theta,S}(x)~=~-[w^\top x+b]_+~.
\]
for $\theta=(w,b)$ where $w\in \reals^d,b\in\reals$.

As in the case of one-Hidden-Layer neural networks, we first need to argue that there exist some choice of parameters $\theta$, so that with a single neuron $F_S(\theta)$ will be meaningfully smaller than a trivial baseline value (in particular, $F_S(\theta)=1$, attained when $\theta=0$). This is more challenging than for one-hidden-layer networks, because unlike such networks, a single neuron cannot express the parity function precisely. However, we show below that there is a choice of parameters so that $F_S(\theta)$ is \emph{polynomially} smaller than $1$. Finding such parameters is a a rather common notion of ``weak-learnability'', considered for example in \citet{abbe2022non} in the context of parities. Thus, we can ask whether one can ``weakly learn'' fixed parities using single ReLU neurons and the squared loss. 
\begin{theorem} \label{thm:single_neuron_positive}
     Let $S\subseteq[d]$ such that $|S|$ is even and at least $4$. Then for the single-neuron network
     \[
    N_{\theta,S}(x)=(-1)^\frac{|S|-2}{2}[w^\top x+b]_+~~~\text{with}~~~ w=\frac{1}{2}|S|^{-3/2}\cdot \mathbf{1}_S~~\text{and}~~b=0~,
     \]
    it holds that
      \begin{equation*}
         F_S(\theta)~=~\E_{x\sim\{\pm1\}^d} \left[\left(N_{\theta,S}(x)-p_S(x)\right)^2\right]~\leq~1-\frac{1}{8|S|^2}~.
       \end{equation*}
\end{theorem}
The proof appears in the appendix, but its main idea is that the single neuron function $x\mapsto [\mathbf{1}_S^\top x]_+$ has a polynomially large correlation with the parity function $p_S$. We show this via Boolean analysis, by relating the $x\mapsto [\mathbf{1}_S^\top x]_+$ function to the majority function, and using the fact that its Fourier coefficients can be polynomially large.

With this theorem in hand, we now turn to state our main result in this section, namely that training a single neuron network using PGD will not reduce $F_S(\theta)$ by more than an exponentially small amount (compared to the trivial value of $1$). In what follows, we use $(w_t,b_t)$ to denote the $t$-th PGD iterate.

\begin{theorem}\label{thm:single}
    Suppose that $d\geq 30$, $|S|\geq \max\{72\ln\left(10d\sigma\sqrt{\sum_{j=0}^T(1-\eta)^{2(T-j)}}\right),6\ln(11(d+1))\}$. Then after $T$ steps of PGD with the quadratic loss function $F_S(\theta)=\E_x(N_\theta(x)-p_S(x))^2$ over the single neuron network $N_\theta(x)=[w^\top x+b]_+$, %
    it holds with probability at least $1-\exp(-|S|/18)$ that
    \begin{equation*}
        F_S(w_t,b_t)=\E_{x\sim\{\pm1\}^d}\left([w_T^\top x+b_T]_+-p_S(x)\right)^2>1-\left(2+T+\frac{\eta\sqrt{T}}{\sigma}\right) e^{-\frac{|S|}{18}}~.
    \end{equation*}
\end{theorem}
Note that the expression $\sum_{j=0}^T(1-\eta)^{2(T-j)}$ can easily be bounded by a constant, under the mild assumption that the step size $\eta$ is smaller than $2$. Aside from this point, and the fact that we consider a single neuron and the quadratic loss, the theorem is quite similar to  \thmref{thm:hard} from the previous section: Namely, as long as the PGD algorithm parameters (the step size $\eta$, the perturbation variance $\sigma^2$ and its inverse, and the number of iterations $T$) are less than exponential in $|S|$, the probability of achieving a non-trivial loss when trying to learn the fixed parity function $p_S$ decays exponentially with $|S|$. Moreover, note that for simplicity, we consider the ReLU neuron $x\mapsto [w^\top x+b]_+$, but a similar result can be readily proven for the negative ReLU neuron $x\mapsto -[w^\top x+b]_+$. Thus, for any moderately large even $|S|$, we see that PGD will fail in providing any meaningful guarantees, even though it is achievable with a suitable choice of parameters by \thmref{thm:single_neuron_positive}.

The proof of the theorem is rather similar in structure to the hardness result from the previous section. The key difference is in handling the squared loss. In more details, recall that our proof technique requires the marginal distribution of the parameters to remain (roughly) spherically symmetric. In particular, fixing $b=0$ for simplicity, our objective function takes the form
\begin{align}
F_S((w,0)) ~&=~ \E_{x\sim \{\pm 1\}^d}\left[\left([w^\top x]_+
-p_S(x)\right)^2\right]\notag\\
&=~ \E_{x\sim \{\pm 1\}^d}\left[\left([w^\top x]_+\right)^2\right]-2\E_{x\sim \{\pm 1\}^d}\left[p_S(x)[w^\top x]_+\right]+1\label{eq:Fw0form}~.
\end{align}
An elementary calculation reveals that the first expectation in the expression above equals $\frac{1}{2}\norm{\bw}^2$. Therefore,
\[
\frac{\partial}{\partial w}F_S((w,0))~=~ w - 2\E_{x\sim \{\pm 1\}^d}\left[p_S(x)\cdot \mathbf{1}_{\{w^\top x>0\}} x\right]~,
\]
so the PGD update at iteration $t$ takes the form
\begin{align*}
w_{t+1} ~&=~ w_t - \eta\left(w_t - 2\E_{x\sim \{\pm 1\}^d}\left[p_S(x)\cdot \mathbf{1}_{\{w^\top x>0\}} x\right]\right)-\xi_t\\
&=~(1-\eta)w_t + 2\eta\E_{x\sim \{\pm 1\}^d}\left[p_S(x)\cdot \mathbf{1}_{\{w^\top x>0\}} x\right]-\xi_t~.
\end{align*}
The second term is similar to the gradient of the expected linear loss (see \eqref{eq:Fderiv}), and we can similarly show that it tends to be exponentially small. If we neglect it, we get that 
\[
w_{t+1}~\approx~(1-\eta)w_t+\xi_t~.
\]
Recalling that $\xi_t$ is a spherically symmetric Gaussian random variable, we get the following: If the marginal distribution of $w_t$ (as a function of the initialization and previous PGD iterates) is (nearly) spherically symmetric, then $w_{t+1}$ will also have a (nearly) spherically symmetric distribution, since it is approximately the sum of two independent spherically symmetric random variables. Thus, the parameters $w_1,w_2,\ldots$ remain (nearly) spherically symmetric, which makes our proof technique applicable as discussed earlier. We note that for this to work, we crucially required the quadratic term in \eqref{eq:Fw0form} to involve a single neuron. 

Finally, we note that our remarks regarding \thmref{thm:hard} apply equally well to \thmref{thm:single}: For example, the proof can be readily modified to handle other spherically symmetric perturbations, and one can trade off between the probability bound and the bound on $F_S$. 

\section{Conclusion and Open Questions}
\label{sec:conclusion}

In this paper, we studied the hardness of learning a fixed parity function, using PGD and two different network/loss settings: The first is one-hidden-layer ReLU networks and linear loss, and the second is a single ReLU neuron network and squared loss. In both cases, we showed that even though there exists a choice of parameters leading to non-trivial performance, perturbed gradient descent will fail in finding it, and the resulting expected loss will remain essentially trivial. Concretely, the hardness of learning the parity functions (in terms of the required number of iterations, magnitude of improvement compared to the trivial loss value, etc.) scales exponentially with the parity size.

Moreover, in order to prove our theorems, we established a new result on the Fourier spectrum of linear threshold (or weighted majority) functions, which may be of independent interest: Namely, whereas the Fourier coefficients can be polynomially large for the plain majority function, they are exponentially small for weighted majorities (in an average-case-sense). 

Our work leaves open several questions and directions for future research. The first obvious question is whether our results can be extended to other network architectures, other loss functions, and other gradient-based methods (at least with standard initializations, which are oblivious to the target function). A first step might be to prove a hardness result for one-hidden-layer ReLU neural networks and the squared loss, which is already beyond the reach of our current proofs. 
Moreover, it would be interesting to extend our results from PGD to stochastic gradient descent (SGD), where one utilizes gradients with respect to specific examples $(x,p_S(x))$ sampled at random, as commonly done in practice. Equivalently, SGD can be viewed as PGD with a certain parameter-dependent, non-spherically symmetric random perturbation (hence our proof techniques do not apply). Due to the results of \citet{abbe2020universality,abbe2021power}, which show that sufficiently contrived network architectures learn parities with SGD, we know that some restrictions on either the gradient method or the network architecture must be posed. However, we conjecture that learning parities with SGD or PGD is hard for an extremely broad class of architectures. Identifying this class is an open problem.  

Another, related open question is whether our bound on the Fourier coefficients of linear threshold functions (\thmref{thm:f0bound}) can be extended beyond spherically-symmetric distributions of $w$. In particular, we conjecture that one can formulate deterministic conditions on $w$, which will imply that the Fourier coefficients of $x\mapsto \mathbf{1}_{\{w^\top x>0\}}$ are exponentially small in $|S|$. For example, these conditions might apply whenever $w$ is not close to certain vectors with a special structure (such as the all-ones vector, corresponding to the majority function). With such an improved result, one could then handle optimization dynamics where the parameters do not have a (nearly) spherically symmetric distribution, and thus possibly analyze more general architectures and methods.

\bibliographystyle{plainnat}
\bibliography{mainbib}

\appendix
\section{Proofs from \secref{sec:layer}}
\subsection{Proof of \thmref{thm:weights}}

    Recall that $\mathbf{1}_S\in \{0,1\}^d$ refers to the vector with entry $1$ at every coordinate in $S$, and $0$ otherwise. In order to express $p_S$, we use the following weights:
    
    \begin{align*}
    (w_j,b_j)_{j=0}^{|S|} &= (-\mathbf{1}_S,|S|+1+2j) \\
    (u_j)_{j=0}^{|S|} &= \begin{cases}
        1 & j=0 \\
        4j(-1)^j & \text{otherwise}~.
    \end{cases}
    \end{align*}  
    This yields the network
    \[
    N_\theta (x)=[-\mathbf{1}_S^\top x+|S|+1]_++\sum_{j=1}^{|S|}4j(-1)^j[-\mathbf{1}_S^\top x+|S|+1-2j]_+~,
    \]
    with the network parameters satisfying the theorem requirements: 
    \begin{align*}
        n &= |S|+1 \\
        ||\theta|| &= \sqrt{\sum_{j=0}^{|S|}\left(||\mathbf{1}_S||+(|S|+1+2j)^2+(4j)^2\right)+1}\leq \sqrt{\sum_{j=1}^{|S|}\left(||S|+10|S|^2+17|S|^2\right)}\leq 6|S|^\frac{3}{2}~.
    \end{align*}
        Note that for any $x$ that satisfies $\mathbf{1}_S^\top x=|S|$ (that is, a vector with $1$ entries in all coordinates in $S$, hence $p_S(x)=1$) we have:
    \begin{equation}
    \label{1_parity}
      N_\theta(x)=[-|S|+|S|+1]_+=1=p_S(x) ~. 
    \end{equation}

    As for the rest of the vectors, for every $x$ that satisfies $\mathbf{1}_S^\top x<|S|$ there exists $i_0(x)=\text{min}\{i\in S:x_i=-1\}$ and we may define the new vector $\hat{x}=x+2e_{i_0}\in\{\pm1\}^d$ where $e_{i_0}$ is the $i_0$-th standard basis vector: In words, $\hat{x}$ is the vector resulting when changing the first negative coordinate of $x$ in $S$ from $-1$ to $+1$. Note that if we perform this change iteratively, the finite sequence $(x,\hat{x},\hat{\hat{x}},....)$ will always end in a vector $y$ that satisfies $\mathbf{1}_S^\top y=|S|$ and $p_S(y)=1$. The corresponding network outputs can be written as $(N_\theta(x),N_\theta(\hat{x}),N_\theta(\hat{\hat{x}}),...,1)$. The proof idea is now to argue that this sequence alternates between $1$ and $-1$, which (by induction, starting from the last entry which equals $1=p_S(y)$) implies that it equals the sequence $(p_s(x),p_s(\hat{x}),p_S(\hat{\hat{x}}),\ldots,1)$. Thus, $N_{\theta}(x)=p_S(x)$. This argument applies for any $x$, establishing that $N_{\theta}$ expresses $p_S$ as required.
    
    A bit more formally, for each $x\in\{\pm 1\}^d$ we define $k_x:=
    \left|\{i\in S:x_i=-1\}\right|$ and we turn to prove that
    \begin{align*}
        N_\theta(x)-N_\theta(\hat{x})=2(-1)^{k_x}~.
    \end{align*}
    This will eventually lead (by induction) to the sequence equality
    \begin{align*}
        (N_\theta(x),N_\theta(\hat{x}),N_\theta(\hat{\hat{x}}),...,1)=(p_S(x),...,1,-1,1)~.
    \end{align*}
    It is easy to see that $p_S(x)=(-1)^{k_x}$ and that $\mathbf{1}_S^\top x=|S|-2k_x$. Additionally, using the notation $\hat{x}$ presented above, for each $x$ that satisfies $\mathbf{1}_S^\top x<|S|$ we have $k_{\hat{x}}=k_x-1$ and hence $\mathbf{1}_S^\top \hat{x}=|S|-2k_x+2$. As a result,
    \begin{align*}
        N_\theta (\hat{x}) &= [-\mathbf{1}_S^\top \hat{x}+|S|+1]_++\sum_{j=1}^{|S|}4j(-1)^j[-\mathbf{1}_S^\top \hat{x}+|S|+1-2j]_+\\
            &= [2k_{\hat{x}}+1]_++\sum_{j=1}^{k_{\hat{x}}}4j(-1)^j[2k_{\hat{x}}+1-2j]_+\\
        N_\theta (x) &=  [2k_{\hat{x}}+3]_++\sum_{j=1}^{k_{\hat{x}}+1}4j(-1)^j[2k_{\hat{x}}+3-2j]_+~.
    \end{align*}
    Combining the two equations above, we can calculate the difference between the two outputs:
    \begin{align*}
        N_\theta (x)-N_\theta (\hat{x}) &= 2
         +\sum_{j=1}^{k_{\hat{x}}}4j(-1)^j \cdot 2+4(k_{\hat{x}}+1)(-1)^{k_{\hat{x}}+1}\cdot 1~.
    \end{align*}
    Recalling that $\sum_{j=1}^{k_{\hat{x}}}j(-1)^j=\lceil\frac{k_{\hat{x}}}{2}\rceil(-1)^{k_{\hat{x}}}$, the above equals
    \begin{align}
         &= 2
         +8\left\lceil\frac{k_{\hat{x}}}{2}\right\rceil(-1)^{k_{\hat{x}}}+4(k_{\hat{x}}+1)(-1)^{k_{\hat{x}}+1}=\nonumber\\
         &= 2+
         \begin{cases}
             (-1)^{k_{\hat{x}}}\left[4k_{\hat{x}}+4-4(k_{\hat{x}}+1)\right] & \text{$k_{\hat{x}}$ is odd}\nonumber \\
             (-1)^{k_{\hat{x}}} \left[4k_{\hat{x}}-4(k_{\hat{x}}+1)\right] & \text{$k_{\hat{x}}$ is even} 
         \end{cases} \\
         \label{parity_diff}
          &=  \begin{cases}
             2 & \text{$k_{\hat{x}}$ is odd} \\
             -2 & \text{$k_{\hat{x}}$ is even} 
         \end{cases}
         =2(-1)^{k_{\hat{x}}+1}=2(-1)^{k_x}~.
    \end{align}

    Now we can complete the proof by showing that $N_\theta (x)=(-1)^{k_x}=p_S(x)$ via induction on $k_x$. We have already seen in \eqref{1_parity} that for $k_x=0$ (i.e. $\mathbf{1}_S^\top x=|S|$) we have $N_\theta(x)=p_S(x)=1$. Continuing the induction, fix an arbitrary $x\in \{\pm 1\}^d$ with $k_x\geq 1$ and assume that for all $x'\in \{\pm 1\}^d$ with $k_{x'}<k_x$ we have $N_\theta(x')=(-1)^{k_{x'}}=p_S(x')$. In particular the assumption holds for $\hat{x}$ since $k_{\hat{x}}=k_x-1$.  Plugging that into \eqref{parity_diff}, we get:
    \begin{align*}
        N_\theta(x)-(-1)^{k_{\hat{x}}}=N_\theta(x)-N_\theta(\hat{x})=2(-1)^{k_{\hat{x}}+1}\nonumber\\
        \Rightarrow~~ N_\theta(x)=(-1)^{k_{\hat{x}}+1}(2-1)=(-1)^{k_x}=p_S(x)~. 
    \end{align*}
    Having established that $N_{\theta}(x)=p_S(x)$ for all $x$, the rest of
    \thmref{thm:weights} immediately follows:
     $$F_S(\theta) := -\E _{x\in\{\pm 1\}^d}\left[N_\theta(x)\cdot p_S(x)\right] = -\E _{x\in\{\pm 1\}^d}\left[p_S(x)^2\right]=-1~.$$

\subsection{Proofs of lemmas used in the proof of \thmref{thm:hard}}
\subsubsection{Proof of \lemref{lem:TV01}}

    Using Pinsker's inequality, we have:
    \begin{align*}
        &TV(P,\tilde{P}) \leq \sqrt{\frac{1}{2} KL(P||\tilde{P})}\\
        &=\sqrt{\frac{1}{2}\sum_{t=1}^T\E_{P\left({\{\theta_i\}_{i=0}^{t-1}}\right)}KL\left(P({\theta_t})|\{\theta_i\}_{i=0}^{t-1}||\tilde{P}({\theta_t})|\{\theta_i\}_{i=0}^{t-1}\right)+\frac{1}{2}KL\left(P({\theta_0})||\tilde{P}({\theta_0})\right)}~,
    \end{align*}
    where KL is the the Kullback-Leibler divergence, and we used the KL-divergence chain rule. 
    Under both distributions, we have $\theta_0\sim\mathcal{N}(0,\sigma^2I)$ and thus $KL(P({\theta_0})||\tilde{P}({\theta_0}))=0$. Also, conditioned on $\{\theta_i\}_{i=0}^{t-1}$, $\theta_t$ has a Gaussian distribution with mean $\theta_{t-1}-\eta\nabla F_S(\theta_{t-1})$ and variance $\sigma^2$ under $P$, and a Gaussian distribution with mean $\theta_{t-1}-\eta\left[\nabla F_S(\theta_{t-1})\right]_\varepsilon$ and variance $\sigma^2$ under $\tilde{P}$. Thus, using a standard formula for the KL divergence between two random variables, it follows that the displayed equation above equals:
    \begin{align*}
        &\sqrt{\frac{1}{2}\sum_{t=1}^T\E_{P\left({\{\theta_i\}_{i=0}^{t-1}}\right)}KL\left(\theta_{t-1}+\Delta_{t-1}|\{\theta_i\}_{i=0}^{t-1}||\theta_{t-1}+\tilde{\Delta}_{t-1}|\{\theta_i\}_{i=0}^{t-1}\right)} \\
        &=\sqrt{\frac{1}{2}\sum_{t=1}^T\E_{P\left({\{\theta_i\}_{i=0}^{t-1}}\right)}KL\left(\theta_{t-1}-\eta\nabla F_S(\theta_{t-1}) +\xi_{t-1}|\{\theta_i\}_{i=0}^{t-1}||\theta_{t-1}-\eta[\nabla F_S(\theta_{t-1})]_\varepsilon +\xi_{t-1}|\{\theta_i\}_{i=0}^{t-1}\right)}
    \end{align*}
Since both $\theta_{t-1}$ and $\eta\nabla F_S(
\theta_{t-1})$ are fixed under the conditioning of $\{\theta_i\}_{i=0}^{t-1}$, our only random element in the expression above is $\xi_{t-1}$. Furthermore,  since $\xi_{t-1}\sim\mathcal{N}(0,\sigma^2I)$, we get that
\begin{align*}
    \theta_{t-1}-\eta\nabla F_S(\theta_{t-1}) +\xi_{t-1}|\{\theta_i\}_{i=0}^{t-1}&~~\sim~~\mathcal{N}(\theta_{t-1}-\eta\nabla F_S(\theta_{t-1}),\sigma^2I)\\
    \theta_{t-1}-\eta\left[\nabla F_S(\theta_{t-1})\right]_\varepsilon +\xi_{t-1}|\{\theta_i\}_{i=0}^{t-1}&~~\sim~~\mathcal{N}(\theta_{t-1}-\eta\left[\nabla F_S(\theta_{t-1})\right]_\varepsilon,\sigma^2I).
\end{align*}
     The KL divergence of two Gaussian variables with means $\mu_1,\mu_2$ and covariance matrix $\sigma^2I$ is well-known to equal $\frac{\norm{\mu_1-\mu_2}^2}{2\sigma^2}$. Plugging this into the displayed equation above, we get that it equals
     \begin{align*}
         &\sqrt{\frac{1}{2}\sum_{t=1}^T\E_{P\left(\{\theta_i\}_{i=0}^{t-1}\right)}\left[\frac{\norm{\theta_{t-1}-\eta\nabla F_S(\theta_{t-1})-\left(\theta_{t-1}-\eta[\nabla F_S(\theta_{t-1})]_\varepsilon\right)}^2}{2\sigma^2}\right]}\\
         &=\sqrt{\frac{1}{2}\sum_{t=1}^T\E_{P\left(\{\theta_i\}_{i=0}^{t-1}\right)}\left[\frac{\norm{
         \eta\left(\nabla F_S(\theta_{t-1})-[\nabla F_S(\theta_{t-1})]_\varepsilon\right)}^2}{2\sigma^2}\right]}\leq\sqrt{\frac{1}{2}\sum_{t=1}^T\E_{P\left({\{\theta_i\}_{i=0}^{t-1}}\right)}\frac{\eta^2\varepsilon^2}{2\sigma^2}}\\
         &=\frac{\eta\varepsilon\sqrt{T}}{2\sigma}~.
     \end{align*}

\subsubsection{Proof of \lemref{lem:gradsmall}} \label{sec:gradsmallproof}

As $S$ is fixed, we denote $F_S$ as $F$ for simplicity. Note that the partial derivatives satisfy 
\begin{align*}
\partial_{u_{j}}F(\theta) &=-\mathbb{E}_{x\in\left\{ \pm1\right\} ^{d}}\left[\left[w_{j}^{T}x+b_{j}\right]_{+}\prod_{i\in S}x_{i}\right]\\
\partial_{w_{j}}F(\theta) &=u_{j}\mathbb{E}_{x\in\left\{ \pm1\right\} ^{d}}\left[-1_{\left\{ w_{j}^{T}x+b_{j}>0\right\} }\prod_{i\in S}x_{i}\right]\cdot x \\
\partial_{b_{j}}F (\theta)&=u_{j}\mathbb{E}_{x\in\left\{ \pm1\right\} ^{d}}\left[-1_{\left\{ w_{j}^{T}x+b_{j}>0\right\} }\prod_{i\in S}x_{i}\right]~.
\end{align*}
Since $\norm{ \nabla F(\theta)}=\sqrt{\sum_{j=1}^n\left(\left| \partial_{u_j}F(\theta)\right|^2+\norm{\partial_{w_j} F(\theta)}^2+\left|\partial_{b_j} F(\theta)\right|^2\right)}$, all we have to do is to bound the expected squared norm of the three expressions above using \thmref{thm:fbound_with_b} (stated and proved in the next subsection):
\begin{align*}
    \E_\theta\left(\partial_{b_j}F\right)^2 &= \E_{u_j,w_j,b_j}\left[u_j^2\E_{x \in \{\pm 1\}^d}^2\left[-1_{\{w_j^\top x+b_j>0\}}\prod_{i \in S}x_i\right]\right] \\
    &= \E_{u_j\sim\mathcal{N}(0.\sigma^2)}\left[u_j^2\right]\cdot  \E_{w_j,b_j\sim\mathcal{N}(0,\sigma^2 I_{d+1})} \left[\E_{x \in \{\pm 1\}^d}^2\left[1_{\{w_j^\top x+b_j>0\}}\prod_{i \in S}x_i\right]\right] \\
    &= \sigma^2\cdot\E_{w_j,b_j\sim\mathcal{N}(0,\sigma^2 I_{d+1})}f_S^2(w_j,b_j)<8\sigma^2e^{-\frac{|S|}{4}}~,
\end{align*}
where the first transition is due to the independence of $u_j$ from $w_j,b_j$ and the last transition is by \thmref{thm:fbound_with_b}.

Turning to the next partial derivative of $F$, we have
\begin{align}
    \E_\theta\norm{\partial_{w_j}F}^2 &= \E_{u_j,w_j,b_j}\left[\sum_{k=1}^du_j^2\E_{x \in \{\pm 1\}^d}^2\left[-1_{\{w_j^\top x+b_j>0\}}x_k\prod_{i \in S}x_i\right]\right] \nonumber\\
    &= \sum_{k=1}^d\E_{u_j\sim\mathcal{N}(0,\sigma^2)}\left[u_j^2\right]\cdot\E_{w_j,b_j\sim\mathcal{N}(0,\sigma^2)}\left[\E_{x \in \{\pm 1\}^d}^2\left[1_{\{w_j^\top x+b_j>0\}}x_k\prod_{i \in S}x_i\right]\right]\nonumber\\
    &= \sum_{k=1}^d\sigma^2\E_{w_j,b_j}\left[\begin{cases}
        \E_x^2[1_{\{w_j^\top x+b_j\}}\prod_{i\in S\backslash\{k\}}x_i] & k\in S\\
        \E_x^2[1_{\{w_j^\top x+b_j\}}\prod_{i\in S\cup\{k\}}x_i] & k \notin S
    \end{cases} \right] \nonumber\\     &=\sum_{k=1}^d\sigma^2\begin{cases}
        \E_{w_j,b_j}f_{S\backslash\{k\}}^2(w_j,b_j) & k\in S\nonumber\\
        \E_{w_j,b_j}f_{S\cup\{k\}}^2(w_j,b_j) & k \notin S
    \end{cases} \nonumber\\
    &< \sum_{k=1}^d\sigma^2\begin{cases}
        \frac{6e^{2.3}}{\sqrt{2}\pi^\frac{3}{2}}e^{-\frac{|S|-1}{4}} & k\in S\\
        \frac{6e^{2.3}}{\sqrt{2}\pi^\frac{3}{2}}e^{-\frac{|S|+1}{4}} & k \notin S
    \end{cases}\nonumber \\
    \label{eq:partial_derivative_w_norm_bound}&\leq d\sigma^2\frac{6e^{2.55}}{\sqrt{2}\pi^\frac{3}{2}}e^{-\frac{|S|}{4}}<10d\sigma^2e^{-\frac{|S|}{4}}~,
\end{align}
where again, the first transition is due to the independence of $u_j$ from $w_j,b_j$, and the transition in the fifth line is due to \thmref{thm:fbound_with_b}.

Turning to the last partial derivative of $F$, we have
\begin{align}
    \E_\theta\left|\partial_{u_j}F\right| &= \E_{u_j,w_j,b_j}\left|\E_x[w_j^\top x+b_j]_+\prod_{i \in S}x_i\right| \leq \E_{u_j,w_j,b_j}\left|\E_x\left[(w_j^\top x+b_j)1_{\{w_j^\top x+b_j>0\}}\prod_{i \in S}x_i\right]\right|\nonumber\\
    &\leq \E_{w_j,b_j}\left|\E_x\left[b_j1_{\{w_j^\top x+b_j>0\}}\prod_{i \in S}x_i\right]\right|+\E_{w_j,b_j}\left|w_j^\top \E_x\left[x1_{\{w_j^\top x+b_j>0\}}\prod_{i \in S}x_i\right]\right|\nonumber\\
    &\leq \sqrt{\E_{w_j,b_j}b_j^2\cdot\E_{w_j,b_j}\E_x^2\left[1_{\{w_j^\top x+b_j>0\}}\prod_{i \in S}x_i\right]}+\E_{w_j,b_j}\left[\norm{w_j}\left|\left|\E_x\left[x1_{\{w_j^\top x+b_j>0\}}\prod_{i \in S}x_i\right]\right|\right|\right]\nonumber\\
    &\leq \sqrt{\sigma^2\cdot\E_{w_j,b_j}f_S^2(w_j,b_j)}+\sqrt{\E_{w_j,b_j}\norm{w_j}^2\cdot\E_{w_j,b_j}\left|\left|\E_x\left[x1_{\{w_j^\top x+b_j>0\}}\prod_{i \in S}x_i\right]\right|\right|^2}\nonumber\\
    &\leq \sqrt{\sigma^2 \cdot 4e^{-\frac{|S|}{4}}}+\sqrt{d\sigma^2\cdot\sum_{k=1}^d\begin{cases}
    \E_{w_j,b_j}\E_x^2 \left[1_{\{w_j^\top x+b_j>0\}}\prod_{i \in S\cup\{k\}}x_i\right] & k \notin S\\
    \E_{w_j,b_j}\E_x^2 \left[1_{\{w_j^\top x+b_j>0\}}\prod_{i \in S\backslash\{k\}}x_i\right] & k\in S
    \end{cases}}\nonumber\\
    &=2\sigma e^{-\frac{|S|}{8}}+\sqrt{d\sigma^2\cdot\sum_{k=1}^d\begin{cases}
    \E_{w_j,b_j}f_{S\cup\{k\}}^2(w_j,b_j) & k \notin S\\
    \E_{w_j,b_j}f_{S\backslash\{k\}}^2(w_j,b_j) & k\in S
    \end{cases}}\nonumber\\
    &\leq 2\sigma e^{-\frac{|S|}{8}}+\sqrt{d\sigma^2\cdot\sum_{k=1}^d \frac{6e^{2.55}}{\sqrt{2}\pi^\frac{3}{2}}e^{-\frac{|S|}{4}}}\leq 2\sigma e^{-\frac{|S|}{8}}+\sqrt{10d^2\sigma^2 e^{-\frac{|S|}{4}}} \leq 5d\sigma e^{-\frac{|S|}{8}}\leq e^{-\frac{|S|}{9}}~,  
\end{align}
where the transitions in both the fifth line and the last line are due to \thmref{thm:fbound_with_b}. Combining these three bounds on the partial derivatives of $F$, we get
\begin{align*}
    \E_\theta&\left[\norm{\nabla F}\right] \leq \E_\theta\left[\sum_{j=1}^n\sqrt{\left|\partial_{u_j}F\right|^2+\norm{\partial_{w_j}F}^2+\left|\partial_{b_j}F\right|^2} \right]\\
    &\leq\E_\theta\left[\sum_{j=1}^n\left(\sqrt{\left|\partial_{u_j}F\right|^2}+\sqrt{\norm{\partial_{w_j}F}^2+\left|\partial_{b_j}F\right|^2}\right) \right] =\sum_{j=1}^n\left(\E_\theta\left|\partial_{u_j}F\right|+\E_\theta\sqrt{\norm{\partial_{w_j}F}^2+\left|\partial_{b_j}F\right|^2}\right) \\
    &\leq \sum_{j=1}^n\left(\E_\theta\left|\partial_{u_j}F\right|+\sqrt{\E_\theta\norm{\partial_{w_j}F}^2+\E_\theta\left|\partial_{b_j}F\right|^2}\right)  \leq \sum_{j=1}^n\left(5d\sigma e^{-\frac{|S|}{8}}+\sqrt{10d\sigma^2e^{-\frac{|S|}{4}}+8\sigma^2e^{-\frac{|S|}{4}}}\right) \\
    &\leq \sum_{j=1}^n\left(5d\sigma e^{-\frac{|S|}{8}}+5\sqrt{d}\sigma e^{-\frac{|S|}{8}}\right)\leq 6nd\sigma e^{-\frac{|S|}{8}}\leq e^{-\frac{|S|}{9}}~.
\end{align*}
In the above, the first and second transitions use the fact that $\sqrt{a+b}\leq\sqrt{a}+\sqrt{b}$ for non-negative numbers $a,b$, in the third line we used Jensen's inequality with regard to the square root function, and the last transition used the lemma requirements for $|S|\geq72 \ln(6nd\sigma)$.
By Markov's inequality, we get
\begin{align*}
     \text{Pr}_\theta \left(\norm{\nabla F_S (\theta)}\geq e^{-\frac{|S|}{18}}\right)\leq\frac{e^{-\frac{|S|}{9}}}{e^{-\frac{|S|}{18}}}=e^{-\frac{|S|}{18}}~.\\
\end{align*}

\subsubsection{Proof of \lemref{lem:randomF_small}}

    Let us simplify the expression of $F_S(\theta)$:
    \begin{align*}
        F_S(\theta)&=-\E_{x\{\pm1\}^d}\left[\prod_{i\in S}x_i\sum_{j=1}^nu_j[w_j^\top x+b_j]_+\right]=-\sum_{j=1}^nu_j\E_{x\{\pm1\}^d}\left[\prod_{i\in S}x_i[w_j^\top x+b_j]_+\right]\\
        &=\sum_{j=1}^nu_j\partial_{u_j} F_S(\theta)~.
    \end{align*}
According to \lemref{lem:gradsmall} $\partial_{u_j}F_S$ satisfies $\E_\theta\left|\partial_{u_j}F\right| \leq 5d\sigma\sqrt{T+1}e^{-\frac{|S|}{8}} $. In addition, for all $j\leq n$, $u_j$ and $\partial_{u_j}F_S(\theta_T)=-\E_{x\in\{\pm 1\}^d}\left[\prod_{i\in S}x_i[w_j^\top x+b_j]_+\right]$ are actually independent and thus
\begin{align*}
    &\E_{\theta}\left|F_S(\theta)\right| \leq \sum_{j=1}^n\E_{\theta}\left[\left|u_j\right|\left|\partial_{u_j}F_S(\theta)\right|\right]\\ 
    &\stackrel{(1)}{=}\sum_{j=1}^n\E_{\theta}\left[\left|u_j\right|\right]\cdot\E_{\theta}\left[\left|\partial_{u_j}F_S(\theta)\right|\right]\\
    &\stackrel{(2)}{\leq} n\sigma\sqrt{\frac{2(T+1)}{\pi}}5d\sigma\sqrt{T+1} e^{-\frac{|S|}{8}}\\
    &\stackrel{(3)}{\leq} 5nd\sigma^2(T+1)e^{-\frac{|S|}{8}}\stackrel{(3)}{\leq}e^{-\frac{|S|}{9}}~,
\end{align*}
where $(1)$ is due to the independence of $u_j$ and $\partial_{u_j}F_S(\theta_T)$, $(2)$ is derived from substituting the bound for $\E_{\theta}\left|\partial_{u_j} F_S(\theta)\right|$ using \lemref{lem:gradsmall}, and $(3)$ follows from $|S|\geq72\ln\left(5nd(T+1)\sigma^2\right) $ as assumed in the lemma statement.
We complete our proof using Markov's inequality:
\begin{equation*}
    Pr\left(\left|F_S(\theta)\right|\geq e^{-\frac{|S|}{18}} \right)\leq \frac{e^{\frac{-|S|}{9}}}{e^{\frac{-|S|}{18}}}=e^{\frac{-|S|}{18}}~.
\end{equation*}

\subsection{The Fourier Coefficients of Linear Threshold Functions}

\subsubsection{Proof of \thmref{thm:f0bound}}

By Fubini's theorem, we can switch the order of the expectations
and get:
\begin{align*}
\mathbb{E}_{w\sim\mathcal{N}\left(0,\sigma^2\right)}f_S^{2}(w,0) &=    \mathbb{E}_{w\sim\mathcal{N}\left(0,\sigma^2\right)}\left[\mathbb{E}_{x\in\left\{ \pm1\right\} ^{d}}\left(1_{\left\{ w^{T}x>0\right\} }\prod_{i \in S}x_{i}\right)\cdot\mathbb{E}_{y\in\left\{ \pm1\right\} ^{d}}\left(1_{\left\{ w^{T}y>0\right\} }\prod_{i \in S}y_{i}\right)\right] \\
&= \mathbb{E}_{x,y\in\left\{ \pm1\right\} ^{d}}\left[\mathbb{E}_{w\sim\mathcal{N}\left(0,\sigma^2\right)}\left(1_{\left\{ w^{T}y>0\right\} }\cdot1_{\left\{ w^{T}x>0\right\} }\prod_{i \in S}x_{i}y_{i}\right)\right]\\
&= \mathbb{E}_{x,y\in\left\{ \pm1\right\} ^{d}}\prod_{i \in S}x_{i}y_{i}\left[\mathbb{E}_{w\sim\mathcal{N}\left(0,\sigma^2\right)}\left(1_{\left\{ \frac{w^{T}}{\left\Vert w\right\Vert }\cdot\frac{y}{\sqrt{d}}>0\right\} }\cdot1_{\left\{ \frac{w^{T}}{\left\Vert w\right\Vert }\cdot\frac{x}{\sqrt{d}}>0\right\} }\right)\right]~.
\end{align*}
Letting $H_{\frac{x}{\sqrt{d}}}:= \{w\in S^{d-1}:w^\top x> 0 \}$ denote
the hemisphere centered around $\frac{x}{\sqrt{d}}$
on the unit sphere $S^{d-1}$, this equals
\begin{align*}
    &\mathbb{E}_{x,y\in\left\{ \pm1\right\} ^{d}}\left[\prod_{i \in S}x_{i}y_{i}\mathbb{E}_{w\sim\mathcal{N}\left(0,\sigma^2\right)}\left(1_{\left\{ \frac{w}{\left\Vert w\right\Vert }\in H_{\frac{x}{\sqrt{d}}}\wedge\frac{w}{\left\Vert w\right\Vert }\in H_{\frac{y}{\sqrt{d}}}\right\} }\right)\right] \\
    &= \mathbb{E}_{x,y\in\left\{ \pm1\right\} ^{d}}\left[\prod_{i \in S}x_{i}y_{i}\mathbb{P}_{ w \sim\mathcal{N}\left(0,\sigma^2\right)}\left(\frac{w}{\left\Vert w\right\Vert }\in H_{\frac{x}{\sqrt{d}}}\cap H_{\frac{y}{\sqrt{d}}}\right)\right] \\
    &= \mathbb{E}_{x,y\in\left\{ \pm1\right\} ^{d}}\left[\prod_{i \in S}x_{i}y_{i}\frac{Vol\left(H_{x}\cap H_{y}\right)}{Vol\left(S^{d-1}\right)}\right]~.
\end{align*}

Using \lemref{lem:hemispheres}, the above equals:

\begin{align*}
&= \mathbb{E}_{x,y\in\left\{ \pm1\right\} ^{d}}\left[\prod_{i \in S}x_{i}y_{i}\left(\frac{\pi-\arccos\left(\frac{x^{T}y}{d}\right)}{2\pi}\right)\right] \\
&=\frac{1}{2} \underbrace{\mathbb{E}_{x,y\in\left\{ \pm1\right\} ^{d}}\left[\prod_{i \in S}x_{i}y_{i}\right]}_{=0}-\frac{1}{2\pi}\mathbb{E}_{x,y\in\left\{ \pm1\right\} ^{d}}\left[\prod_{i \in S}x_{i}y_{i}\arccos\left(\frac{x^\top y}{d}\right)\right] \\
&= -\frac{1}{2\pi}\mathbb{E}_{x,y\in\left\{ \pm1\right\} ^{d}}\left[\prod_{i \in S}x_{i}y_{i}\arccos\left(\frac{x^\top y}{d}\right)\right]~.
\end{align*}

$\arccos$ is an analytic function, whose Taylor expansion is given by 
\[
\arccos\left(z\right)=\frac{\pi}{2}-\sum_{j=1}^{\infty}\frac{\left(2j\right)!}{2^{2j}\left(j!\right)^2}\frac{z^{2j+1}}{2j+1}=\frac{\pi}{2}-\sum_{j=1}^{\infty}\alpha_{j}z^{j}~~
\text{where}~~ \alpha_{j}=\begin{cases}
    \frac{\left(j-1\right)!}{2^{j-1}\left[\left(\frac{j-1}{2}\right)!\right]^2}\frac{1}{j} & j \text{ is odd} \\
    0 & j \text{ is even}~.
\end{cases}
\]
Inserting this into the previous expression, we get
\begin{align}
    -\frac{1}{4}\underbrace{\E_{x,y \in \{\pm 1\}^d}\left[\prod_{i \in S}x_iy_i\right]}_{=0}&+\frac{1}{2\pi}\sum_{j=1}^{\infty}\alpha_{j}\mathbb{E}_{x,y\in\left\{ \pm1\right\} ^{d}}\left[\prod_{i\in S}x_{i}y_{i}\left(\frac{x^{T}y}{d}\right)^{j}\right] \nonumber \\
    \label{eq:low_coef} =\frac{1}{2\pi}\sum_{j=1}^{|S|-1}\alpha_{j}\mathbb{E}_{x,y\in\left\{ \pm1\right\} ^{d}}&\left[\prod_{i \in S}x_{i}y_{i}\left(\frac{x^{T}y}{d}\right)^{j}\right]\\
    \label{eq:high_coef}
    &+ \frac{1}{2\pi}\sum_{j=|S|}^{\infty}\alpha_{j}\mathbb{E}_{x,y\in\left\{ \pm1\right\} ^{d}}\left[\prod_{i \in S}x_{i}y_{i}\left(\frac{x^{T}y}{d}\right)^{j}\right]~.
\end{align}

Let us deal with low-order terms in \eqref{eq:low_coef}  first:

\begin{align*}
   \sum_{j=1}^{|S|-1}\frac{\alpha_{j}}{d^{j}}\mathbb{E}_{x,y\in\left\{ \pm1\right\} ^{d}}\left[\prod_{i \in S} x_{i}y_{i}\left(\sum_{i=1}^{d}x_{i}y_{i}\right)^{j}\right] &= \sum_{j=1}^{|S|-1}\frac{\alpha_{j}}{d^{j}}\mathbb{E}_{x,y\in\left\{ \pm1\right\} ^{d}}\left[\prod_{i \in S}x_{i}y_{i}\sum_{r_{1},..,r_{j}=1}^{d}\prod_{k=1}^{j}x_{r_{k}}y_{r_{k}}\right]\\
    &= \sum_{j=1}^{|S|-1} \frac{\alpha_{j}}{d^{j}} \sum_{r_{1},..,r_{j}=1}^{d}\mathbb{E}_{x,y\in\left\{ \pm1\right\} ^{d}}\left[\prod_{i \in S}x_{i}y_{i}\cdot\prod_{k=1}^{j}x_{r_{k}}y_{r_{k}}\right]~.
\end{align*}

For each $j\leq |S|-1$ and each $r_{1},...,r_{j}\in\left\{ 1,...,d\right\} $
there is some $l\in S \backslash\left\{ r_{1},...,r_{j}\right\} $.
Using that and the independence of $x_i,y_i$, we are able to zero out the equation above, and get

\begin{equation}
   \label{eq:low_is_zero} \sum_{j=1}^{|S|-1}\sum_{r_{1},..,r_{j}=1}^{d}\frac{\alpha_{j}}{d^{j}}\left(\underbrace{\mathbb{E}_{x_{l},y_{l}\in\left\{ \pm1\right\} }\left(x_{l}y_{l}\right)}_{=0}\cdot\mathbb{E}_{x,y\in\left\{ \pm1\right\} ^{d-1}}\left[\prod_{i\in S\backslash\{l\}}x_{i}y_{i}\cdot\prod_{k=1}^{j}x_{r_{k}}y_{r_{k}}\right]\right)=0~.
\end{equation}

Thus, we remain only with the higher-order coefficients (in \eqref{eq:high_coef}). By Jensen's inequality they are at most

\begin{align}
    \frac{1}{2\pi}\sum_{j=|S|}^{\infty}\alpha_{j}\mathbb{E}_{x,y\in\left\{ \pm1\right\} ^{d}}\left[\prod_{i \in S}x_{i}y_{i}\left(\frac{x^{T}y}{d}\right)^{j}\right] &\leq \left|\frac{1}{2\pi}\sum_{j=|S|}^{\infty}\alpha_{j}\mathbb{E}_{x,y\in\left\{ \pm1\right\} ^{d}}\left[\prod_{i \in S}x_{i}y_{i}\left(\frac{x^{T}y}{d}\right)^{j}\right]\right| \nonumber \\
    \label{eq:before_alpha_substitute}
    \leq \frac{1}{2\pi}\sum_{j=|S|}^{\infty}\alpha_{j}\mathbb{E}_{x,y\in\left\{ \pm1\right\} ^{d}}\left|\prod_{i\in S}x_{i}y_{i}\left(\frac{x^{T}y}{d}\right)^{j}\right| &=\mathbb{E}_{y\in\left\{ \pm1\right\} ^{d}}\frac{1}{2\pi}\sum_{j=|S|}^{\infty}\alpha_{j}\mathbb{E}_{x\in\left\{ \pm1\right\} ^{d}}\left[\left(\frac{\left|x^{T}y\right|}{d}\right)^{j}\right]~.
\end{align}

 By a standard Stirling approximation bound, it holds for all $j> 1$ that $\sqrt{2\pi j}\left(\frac{j}{e}\right)^je^{\frac{1}{12j}-\frac{1}{360j^3}}<j!<\sqrt{2\pi j}\left(\frac{j}{e}\right)^je^{\frac{1}{12j}}$. Therefore, we can upper bound $\alpha_j$ as follows:

\begin{align*}
    \alpha_j &\leq \frac{\left(j-1\right)!}{2^{j-1}\left[\left(\frac{j-1}{2}\right)!\right]^2}\cdot\frac{1}{j} ~<~ \frac{1}{j}\cdot\frac{2\cdot2^{j-1}(j-1)^{j-1}e^\frac{1}{12(j-1)}}{2^{j-1}\sqrt{2\pi (j-1)}(j-1)^{j-1}e^{\frac{4}{12(j-1)}-\frac{16}{360(j-1)^3}}} \\
    &= \frac{1}{j}\cdot\frac{\sqrt{2}e^{\frac{2}{45(j-1)^3}-\frac{3}{12(j-1)}}}{\sqrt{\pi (j-1)}}<\sqrt{\frac{2}{\pi}}\cdot\frac{2e}{j^\frac{3}{2}}~.
\end{align*}

Using this result we can bound the sum of the $\arccos$ Taylor coefficients:
\begin{equation}
    \label{eq:alpha_sum}\sum_{j=|S|}^{\infty}\alpha_j\leq\sum_{j=|S|}^{\infty}\sqrt{\frac{2}{\pi}}\frac{2e}{j^\frac{3}{2}}\leq\sum_{j=1}^{\infty}\sqrt{\frac{2}{\pi}}\frac{2e}{j^\frac{3}{2}}\leq6e\sqrt{\frac{2}{\pi}}
\end{equation}

Let us fix some $\beta>1$ that will be determined later. For a fixed
$y$ we have that $x^{T}y$ is a sum of independent Rademacher random variables,
and hence by Hoeffding's inequality,
\begin{equation*}
\mathbb{P}\left(\left|x^{T}y\right|\geq\frac{d}{\beta}\,|y\right)\leq2e^{-\frac{2d^{2}}{\beta^{2}\sum_{i=1}^{d}2^{2}}}=2e^{-\frac{d}{2\beta^{2}}}~.
\end{equation*}
Using this inequality and \eqref{eq:alpha_sum}, we can upper bound \eqref{eq:before_alpha_substitute} as follows:

\begin{align*}
    &\mathbb{E}_{y\in\left\{ \pm1\right\} ^{d}}\frac{1}{2\pi}\sum_{j=|S|}^{\infty}\alpha_{j}\mathbb{E}_{x\in\left\{ \pm1\right\} ^{d}}\left[\left(\frac{\left|x^{T}y\right|}{d}\right)^{j}\right]=\\
    &=\mathbb{E}_{y\in\left\{ \pm1\right\} ^{d}}\frac{1}{2\pi}\sum_{j=|S|}^{\infty}\alpha_{j}
    \left( \mathbb{E}_{x\in\left\{ \pm1\right\} ^{d}|\left|x^{T}y\right|\geq\frac{d}{\beta}}\left[\left(\frac{\left|x^{T}y\right|}{d}\right)^{j}\right]\mathbb{P}\left(\left|x^{T}y\right|\geq\frac{d}{\beta}|y\right)\right.\\
    &\quad \quad \quad \quad \quad \left.+ \mathbb{E}_{x\in\left\{ \pm1\right\} ^{d}|\left|x^{T}y\right|<\frac{d}{\beta}}\left[\left(\frac{\left|x^{T}y\right|}{d}\right)^{j}\right]\mathbb{P}\left(\left|x^{T}y\right|<\frac{d}{\beta}|y\right)\right)\\
    &\leq\mathbb{E}_{y\in\left\{ \pm1\right\} ^{d}}\frac{1}{2\pi}\sum_{j=|S|}^{\infty}\alpha_{j}\left(1\cdot2e^{-\frac{d}{2\beta^{2}}}+\left(\frac{1}{\beta}\right)^{j}\cdot1\right)\leq\frac{3e}{\sqrt{2}\pi^\frac{3}{2}}\left(2e^{-\frac{d}{2\beta^{2}}}+\left(\frac{1}{\beta}\right)^{|S|}\right)~.
\end{align*}
Taking $\beta=1.3$,
we get:
\begin{align}
    \label{eq:expect_w_bound} & \mathbb{E}_{w\sim\mathcal{N}\left(0,\sigma^2\cdot I_{d}\right)}\left(f_S^{2}\right) \leq \frac{3e}{\sqrt{2}\pi^\frac{3}{2}}\left(2e^{-\frac{d}{2(1.3)^{2}}}+e^{-\ln(1.3)|S|}\right) \nonumber \\
    &< \frac{6e}{\sqrt{2}\pi^\frac{3}{2}}\left(e^{-\frac{|S|}{3.38}}+e^{-\frac{|S|}{4}}\right) < \frac{12e}{\sqrt{2}\pi^\frac{3}{2}}e^{-\frac{|S|}{4}}\nonumber~,
\end{align}
which can be upper bounded by $6e^{-\frac{|S|}{4}}$.

\subsubsection{The Fourier Coefficients of Biased Linear Threshold Functions}

As discussed in the main paper, we state and prove here an extension of \thmref{thm:f0bound} for the case of linear threshold functions with biases.
\begin{theorem}\label{thm:fbound_with_b}
Suppose that $d\geq 2$, and fix some $S\subseteq [d]$ such that $|S|\geq 2$. Then the function   $f_S(w,b):=\E_{x\sim\{\pm1\}^d}\left[p_S(x)\mathbf{1}_{\{w^\top x+b>0\}}\right]$ satisfies
\begin{equation*}
\mathbb{E}_{(w,b)\sim\mathcal{N}\left(0,\sigma^2\cdot I_{d+1}\right)}\left[f^{2}(w,b)\right]\leq \frac{6e^{2.3}}{\sqrt{2}\pi^\frac{3}{2}}e^{-\frac{|S|}{4}}<8e^{-\frac{|S|}{4}}~.
\end{equation*}
\end{theorem}

\begin{proof}
The proof of this theorem is almost identical to the proof of \thmref{thm:f0bound} with slight changes due to the addition of the parameter $b$. Using the notation $u=(w^\top \, b)$ we get:
\begin{align*}
    f_S(w,b) &=\E_{x \in \{ \pm 1 \}^d} \left[ 1_{\{w^\top x+b>0\}}\prod_{i \in S}x_i\right]=\E_{x \in \{ \pm 1 \}^d} \left[ 1_{\left\{(w^\top \, b)\begin{pmatrix}
        x \\ 1
    \end{pmatrix}>0\right\}}\prod_{i \in S}x_i\right] \\
    &= \E_{x \in \{ \pm 1 \}^d\times\{ 1\}} \left[ 1_{\{u^\top x>0\}}\prod_{i \in S}x_i\right]
\end{align*}

Now we can repeat the former proof of \thmref{thm:f0bound} with slight changes. By using Fubini's theorem, we can switch the order of the expectations
and get:
\begin{align*}
\mathbb{E}_{u\sim\mathcal{N}\left(0,\sigma^2I_{d+1}\right)}f_S^{2} &=    \mathbb{E}_{u\sim\mathcal{N}\left(0,\sigma^2I_{d+1}\right)}\left[\mathbb{E}_{x\in\left\{ \pm1\right\} ^{d}\times\{1\}}\left(1_{\left\{ u^{T}x>0\right\} }\prod_{i \in S}x_{i}\right)\cdot\mathbb{E}_{y\in\left\{ \pm1\right\} ^{d}\{1\}}\left(1_{\left\{ u^{T}y>0\right\} }\prod_{i \in S}y_{i}\right)\right] \\
&= \mathbb{E}_{u,x,y}\left[1_{\left\{ u^{T}y>0\right\} }\cdot1_{\left\{ u^{T}x>0\right\} }\prod_{i \in S}x_{i}y_{i}\right]\\
&= \mathbb{E}_{x,y}\left[\prod_{i \in S}x_{i}y_{i}\mathbb{E}_u\left[1_{\left\{ \frac{u^{T}}{\left\Vert u\right\Vert }\cdot\frac{y}{\sqrt{d+1}}>0\right\} }\cdot1_{\left\{ \frac{u^{T}}{\left\Vert u\right\Vert }\cdot\frac{x}{\sqrt{d+1}}>0\right\} }\right]\right]\\
\end{align*}

In order to describe the hemisphere centered around $\frac{x}{\sqrt{d+1}}$
in $S^{d}$ we will use the notation $H_{\frac{x}{\sqrt{d+1}}}:= \{u\in S^{d}:u^\top x\geq 0 \}$ and continue with our calculation:
\begin{align*}
    &= \mathbb{E}_{x,y}\left[\prod_{i \in S}x_{i}y_{i}\mathbb{E}_{u}\left(1_{\left\{ \frac{u}{\left\Vert u\right\Vert }\in H_{\frac{x}{\sqrt{d+1}}}\wedge\frac{u}{\left\Vert u\right\Vert }\in H_{\frac{y}{\sqrt{d+1}}}\right\} }\right)\right] \\
    &= \mathbb{E}_{x,y}\left[\prod_{i \in S}x_{i}y_{i}\mathbb{P}_{u\sim \mathcal{N}\left(0,\sigma^2I_{d+1}\right)}\left(\frac{u}{\left\Vert u\right\Vert }\in H_{\frac{x}{\sqrt{d+1}}}\cap H_{\frac{y}{\sqrt{d+1}}}\right)\right] \\
    &= \mathbb{E}_{x,y\in\left\{ \pm1\right\} ^{d}\times \{1\}}\left[\prod_{i \in S}x_{i}y_{i}\frac{Vol\left(H_{\frac{x}{\sqrt{d+1}}}\cap H_{\frac{y}{\sqrt{d+1}    }}\right)}{Vol\left(S^{d}\right)}\right] \\
\end{align*}

Using \lemref{lem:hemispheres}, we get that this equals

\begin{align*}
&= \mathbb{E}_{x,y\in\left\{ \pm1\right\} ^{d}\times\{1\}}\left[\prod_{i \in S}x_{i}y_{i}\left(\frac{\pi-\arccos\left(\frac{x^{T}y}{d+1}\right)}{2\pi}\right)\right] \\
&=\frac{1}{2} \underbrace{\mathbb{E}_{x,y\in\left\{ \pm1\right\} ^{d}\times\{1\}}\left[\prod_{i \in S}x_{i}y_{i}\right]}_{=0}-\frac{1}{2\pi}\mathbb{E}_{x,y\in\left\{ \pm1\right\} ^{d}\times\{1\}}\left[\prod_{i \in S}x_{i}y_{i}\arccos\left(\frac{x^\top y}{d+1}\right)\right] \\
&= -\frac{1}{2\pi}\mathbb{E}_{x,y}\left[\prod_{i \in S}x_{i}y_{i}\arccos\left(\frac{x^\top y}{d+1}\right)\right]~.
\end{align*}

The Taylor expansion of $\arccos$ is given by $\arccos\left(z\right)=\frac{\pi}{2}-\sum_{j=1}^{\infty}\frac{\left(2j\right)!}{2^{2j}\left(j!\right)^2}\frac{z^{2j+1}}{2j+1}=\frac{\pi}{2}-\sum_{j=1}^{\infty}\alpha_{j}z^{j}$
where $\alpha_{j}=
    \frac{\left(j-1\right)!}{2^{j-1}\left[\left(\frac{j-1}{2}\right)!\right]^2}\frac{1}{j}$ if $j$ is odd, and $0$ otherwise. Plugging this in, we get that the above equals
\begin{align}
    & -\frac{1}{4}\underbrace{\E_{x,y}\left[\prod_{i \in S}x_iy_i\right]}_{=0}+\frac{1}{2\pi}\sum_{j=1}^{\infty}\alpha_{j}\mathbb{E}_{x,y}\left[\prod_{i\in S}x_{i}y_{i}\left(\frac{x^{T}y}{d+1}\right)^{j}\right] \nonumber \\
    \label{eq:low_coef_with_b} &=\frac{1}{2\pi}\sum_{j=1}^{|S|-1}\alpha_{j}\mathbb{E}_{x,y}\left[\prod_{i \in S}x_{i}y_{i}\left(\frac{x^{T}y}{d+1}\right)^{j}\right] \\
    \label{eq:high_coef_with_b}
    &\quad \quad \quad \quad + \frac{1}{2\pi}\sum_{j=|S|}^{\infty}\alpha_{j}\mathbb{E}_{x,y}\left[\prod_{i \in S}x_{i}y_{i}\left(\frac{x^{T}y}{d+1}\right)^{j}\right]~.
\end{align}

Let us deal with the low-order coefficients (in \eqref{eq:low_coef_with_b}) first, while remembering that $d+1\notin S$:
\begin{align*}
  \sum_{j=1}^{|S|-1}\frac{\alpha_{j}}{(d+1)^j}\mathbb{E}_{x,y}\left[\prod_{i \in S} x_{i}y_{i}\left(\sum_{i=1}^{d+1}x_{i}y_{i}\right)^{j}\right] &= \sum_{j=1}^{|S|-1}\frac{\alpha_{j}}{(d+1)^{j}}\mathbb{E}_{x,y}\left[\prod_{i \in S}x_{i}y_{i}\sum_{r_{1},..,r_{j}=1}^{d+1}\prod_{k=1}^{j}x_{r_{k}}y_{r_{k}}\right]\\
    &= \sum_{j=1}^{|S|-1} \frac{\alpha_{j}}{(d+1)^{j}} \sum_{r_{1},..,r_{j}=1}^{d+1}\mathbb{E}_{x,y}\left[\prod_{i \in S}x_{i}y_{i}\cdot\prod_{k=1}^{j}x_{r_{k}}y_{r_{k}}\right]~.
\end{align*}

Again, exactly as in the proof of \thmref{thm:f0bound}, for each $j\leq |S|-1$ and each $r_{1},...,r_{j}\in\left\{ 1,...,d\right\} $
there is some $l\in S \backslash\left\{ r_{1},...,r_{j}\right\} $.
Using that and the independence of $x_i,y_i$, we get that the equation above is zero, since it equals
\begin{equation}
   \label{eq:low_is_zero_with_b} \sum_{j=1}^{|S|-1}\sum_{r_{1},..,r_{j}=1}^{d+1}\frac{\alpha_{j}}{(d+1)^{j}}\left(\underbrace{\mathbb{E}_{x_{l},y_{l}\in\left\{ \pm1\right\} }\left(x_{l}y_{l}\right)}_{=0}\cdot\mathbb{E}_{x,y\in\left\{ \pm1\right\} ^{d-1}\times\{1\}}\left[\prod_{i\in S\backslash \{l\}}x_{i}y_{i}\cdot\prod_{k=1}^{j}x_{r_{k}}y_{r_{k}}\right]\right)=0~.
\end{equation}

Thus, we remain only with the higher-order  coefficients (that appear in \eqref{eq:high_coef_with_b}). By Jensen's inequality they are at most
\begin{align}
    -\frac{1}{2\pi}\sum_{j=|S|}^{\infty}\alpha_{j}\mathbb{E}_{x,y}\left[\prod_{i \in S}x_{i}y_{i}\left(\frac{x^{T}y}{d+1}\right)^{j}\right] &\leq \left|\frac{1}{2\pi}\sum_{j=|S|}^{\infty}\alpha_{j}\mathbb{E}_{x,y}\left[\prod_{i \in S}x_{i}y_{i}\left(\frac{x^{T}y}{d+1}\right)^{j}\right]\right| \nonumber \\ \label{eq:before_alpha_substitute_with_b}
    \leq \frac{1}{2\pi}\sum_{j=|S|}^{\infty}\alpha_{j}\mathbb{E}_{x,y}\left|\prod_{i \in S}x_{i}y_{i}\left(\frac{x^{T}y}{d+1}\right)^{j}\right| &=\frac{1}{2\pi}\mathbb{E}_{y}\sum_{j=|S|}^{\infty}\alpha_{j}\mathbb{E}_{x}\left[\left(\frac{\left|x^{T}y\right|}{d+1}\right)^{j}\right]~.
\end{align}
By Stirling's approximation, it holds for all $j> 1$ that $\sqrt{2\pi j}\left(\frac{j}{e}\right)^je^{\frac{1}{12j}-\frac{1}{360j^3}}<j!<\sqrt{2\pi j}\left(\frac{j}{e}\right)^je^{\frac{1}{12j}}$. Hence, 
\begin{align*}
    \alpha_j &\leq \frac{\left(j-1\right)!}{2^{j-1}\left[\left(\frac{j-1}{2}\right)!\right]^2}\frac{1}{j} < \frac{1}{j}\frac{2\cdot2^{j-1}(j-1)^{j-1}e^\frac{1}{12(j-1)}}{2^{j-1}\sqrt{2\pi (j-1)}(j-1)^{j-1}e^{\frac{4}{12(j-1)}-\frac{16}{360(j-1)^3}}} \\
    &= \frac{1}{j}\frac{\sqrt{2}e^{\frac{2}{45(j-1)^3}-\frac{3}{12(j-1)}}}{\sqrt{\pi (j-1)}}<\sqrt{\frac{2}{\pi}}\frac{2e}{j^\frac{3}{2}}~. \\
\end{align*}
Using this result we can bound the sum of the $\arccos$ coefficients:
\begin{equation}
\label{eq:alpha_sum_with_b}\sum_{j=|S|}^{\infty}\alpha_j\leq\sum_{j=|S|}^{\infty}\alpha_j\sqrt{\frac{2}{\pi}}\frac{2e}{j^\frac{3}{2}}\leq\sum_{j=1}^{\infty}\sqrt{\frac{2}{\pi}}\frac{2e}{j^\frac{3}{2}}\leq6e\sqrt{\frac{2}{\pi}}~.
\end{equation}

 Let us fix some $\beta>1$ that will be determined later. For a fixed
$y$, we have that for all $i\leq d$,  $x_iy_i$ is an independent Rademacher random variable, and $x_{d+1}y_{d+1}=1$.
Hence, $\E_{x,y} \left[x^\top y\right]=1$, and by Hoeffding's inequality,
\begin{equation*}
    \mathbb{P}\left(\left|x^{T}y-1\right|\geq\frac{d}{\beta}~\middle|~ y\right)\leq2e^{-\frac{2d^{2}}{\beta^{2}\sum_{i=1}^{d}2^{2}}}=2e^{-\frac{d}{2\beta^{2}}}~.
\end{equation*}

Plugging this inequality and \eqref{eq:alpha_sum_with_b} into \eqref{eq:before_alpha_substitute_with_b}, and using the fact that $d\geq|S|$, we get 
\begin{align*}
    &\frac{1}{2\pi}\mathbb{E}_{y}\sum_{j=|S|}^{\infty}\alpha_{j}\mathbb{E}_{x}\left[\left(\frac{\left|x^{T}y\right|}{d+1}\right)^{j}\right]=\\
    &=\frac{1}{2\pi}\mathbb{E}_{y}\sum_{j=|S|}^{\infty}\alpha_{j}\left(\mathbb{E}_{x|\left|x^{T}y-1\right|\geq\frac{d}{\beta}}\left[\left(\frac{\left|x^{T}y\right|}{d+1}\right)^{j}\right]\mathbb{P}\left(\left|x^{T}y-1\right|\geq\frac{d}{\beta}|y\right)\right.\\
    &\left. \quad \quad \quad \quad \quad \quad \quad +\mathbb{E}_{x|\left|x^{T}y-1\right|<\frac{d}{\beta}}\left[\left(\frac{\left|x^{T}y\right|}{d+1}\right)^{j}\right]\mathbb{P}\left(\left|x^{T}y-1\right|<\frac{d}{\beta}|y\right)\right)\\
    &\leq\frac{1}{2\pi}\mathbb{E}_{y}\sum_{j=|S|}^{\infty}\alpha_{j}\left(1\cdot2e^{-\frac{d}{2\beta^{2}}}+\left(\frac{d}{(d+1)\beta}+\frac{1}{d+1}\right)^{j}\cdot1\right)\\
    &\leq\frac{3e}{\sqrt{2}\pi^\frac{3}{2}}\left(2e^{-\frac{d}{2\beta^2}}+\left(\frac{1}{\beta}\right)^{|S|}\left(1+\frac{\beta}{|S|}\right)^{|S|}\right)
    \leq\frac{3e}{\sqrt{2}\pi^\frac{3}{2}}\left(2e^{-\frac{d}{2\beta^{2}}}+e^\beta\left(\frac{1}{\beta}\right)^{|S|}\right)~.
\end{align*}
Taking $\beta=1.3$ and plugging in the above, we get overall that
\begin{align*}
\mathbb{E}_{w\sim\mathcal{N}\left(0,\sigma^2\cdot I_{d+1}\right)}\left(f_S^{2}\right) &\leq \frac{3e}{\sqrt{2}\pi^\frac{3}{2}}\left(2e^{-\frac{d}{2(1.3)^{2}}}+e^{1.3} e^{-\ln(1.3)|S|}\right) \\
    < \frac{3e^{2.3}}{\sqrt{2}\pi^\frac{3}{2}}\left(e^{-\frac{|S|}{3.38}}+e^{-\frac{|S|}{4}}\right) &< \frac{6e^{2.3}}{\sqrt{2}\pi^\frac{3}{2}}e^{-\frac{|S|}{4}}<8e^{-\frac{|S|}{4}}~.
\end{align*}
\end{proof}

\section{Proofs from \secref{sec:single}}
\subsection{Proof of \thmref{thm:single_neuron_positive}}

    First, we calculate the expected value of $[w^\top x]_+^2$ for some vector $w$, using the fact that $\E_x[w^\top x]_+^2=\E_x[-w^\top x]_+^2$ due to the uniform distribution of $x$ on $\{\pm1\}^d$:
    \begin{align}
        \E_x[w^\top x]_+^2&=\frac{1}{2}\left(\E_x[w^\top x]_+^2+\E_x[w^\top x]_+^2\right)=\frac{1}{2}\left(\E_x[w^\top x]_+^2+\E_x[-w^\top x]_+^2\right)\nonumber\\
        &=\frac{1}{2}\E_x\left[[w^\top x]_+^2+[-w^\top x]_+^2)\right]=\frac{1}{2}\E_x\left[(w^\top x)^2\right]=\frac{1}{2}w^\top \E_x\left[xx^\top \right]w\nonumber\\
        \label{eq:mean_squared_ReLU}&=\frac{1}{2}\norm{w}^2~.
    \end{align}
    With this in mind, we analyze $F_S$:
    \begin{align}
        F_S\left(\frac{1}{2|S|^\frac{3}{2}}\mathbf{1}_S,0\right)&=\E\left[\left((-1)^\frac{|S|-2}{2}\left[\frac{1}{2|S|^\frac{3}{2}}\mathbf{1}_S^\top x\right]_+- p_S(x)\right)^2\right]\nonumber\\
        &=\frac{1}{4|S|^3}\E\left[\mathbf{1}_S^\top x\right]_+^2-2(-1)^\frac{|S|-2}{2}\E\left[\left[\frac{1}{2|S|^\frac{3}{2}}\mathbf{1}_S^\top x\right]_+p_S(x)\right]+1\nonumber\\
        &=\frac{1}{4|S|^3}\cdot\frac{\norm{\mathbf{1}_S}^2}{2}-\frac{(-1)^\frac{|S|-2}{2}}{|S|^\frac{3}{2}}\E\left[\left[\mathbf{1}_S^\top x\right]_+p_S(x)\right]+1\nonumber\\
        &\label{eq:single_neuron_loss}=\frac{1}{8|S|^2}-\frac{(-1)^\frac{|S|-2}{2}}{|S|^\frac{3}{2}}\E\left[\left[\mathbf{1}_S^\top x\right]_+p_S(x)\right]+1~.
    \end{align}

Now we would like to bound the expression $(-1)^\frac{|S|-2}{2}\E\left[\left[\mathbf{1}_S^\top x\right]_+p_S(x)\right]$ from below in order to get an upper bound for $F_S$. In order to do that we will introduce a few terms from Boolean Analysis. For a function $h:\{\pm1\}^d\rightarrow\{\pm1\}$ and a subset $A\subseteq[d]$, the notation $\widehat{h}(A)$ will be used to denote the Fourier coefficient of $h$ with respect to $A$. The discrete derivative of $g$ with respect to coordinate $i$ will be noted as $D_ig$ and is defined as $D_ig(x)=\frac{g(x^{i\rightarrow1})-g(x^{i\rightarrow1})}{2}$ where $x^{i\rightarrow j}=(x_1,...,x_{i-1},j,x_{i+1},x_d)$. Lastly, for each subset $U\subseteq [d]$ we define the matching majority function $\text{Maj}_U:\{\pm1\}^d\rightarrow\{\pm1,0\}$ as $\text{Maj}_U(x)=\sign(\mathbf{1}_U^\top x)$ (with the convention that $\sign(0)=0$).\\ 
We consider the function $g(x)=\left[\mathbf{1}_S^\top x\right]_+$ and its Fourier coefficients. Our expression, $\E\left[\left[\mathbf{1}_S^\top x\right]_+p_S(x)\right]$, is exactly $\hat{g}(S)$, the Fourier coefficient of $g$ with respect to $S$. According to  \cite[Proposition 2.19]{odonnell2021analysisbooleanfunctions}, for each $i\in S$ this is also the Fourier coefficient of $D_ig$ that matches the subset $S\backslash\{i\}$. All together we get that 
\begin{align}
    \label{eq:f_S_Fourier}\E_x\left[\left[\mathbf{1}_S^\top x\right]_+p_S(x)\right]=\hat{g}(S)=\widehat{D_ig}(S\backslash\{i\})~.
\end{align} 
Without loss of generality let us examine $D_1g$, assuming $1\in S$ (otherwise, one can pick any other index $j\in S$ and the calculation will be the same).

\begin{align}
    D_1g(x)&=\frac{[\mathbf{1}_S^\top x^{1\rightarrow1}]_+-[\mathbf{1}_S^\top x^{1\rightarrow-1}]_+}{2}=\nonumber\\
    &=\begin{cases}
        \frac{\mathbf{1}_{S\backslash\{1\}}^\top x+1-(w_{S\backslash\{1\}}^\top x-1)}{2} & \mathbf{1}_{S\backslash\{1\}}^\top x\geq 1 \\
        \frac{\mathbf{1}_{S\backslash\{1\}}^\top x+1-0}{2} & \mathbf{1}_{S\backslash\{1\}}^\top x=0 \\
        0  & \mathbf{1}_{S\backslash\{1\}}^\top x\leq-1
    \end{cases}\nonumber\\
    &=\begin{cases}
        1 & \mathbf{1}_{S\backslash\{1\}}^\top x\geq 1 \\
        \frac{1}{2} & \mathbf{1}_{S\backslash\{1\}}^\top x=0 \\
        0  & \mathbf{1}_{S\backslash\{1\}}^\top x\leq-1
    \end{cases}\nonumber\\
    \label{eq:Dg_calculation} &=\frac{\text{Maj}_{S\backslash\{1\}}(x)+1}{2}~.
\end{align}

  Put together,  \eqref{eq:f_S_Fourier} and \eqref{eq:Dg_calculation} imply that in order to bound $\E_x\left[\left[\mathbf{1}_S^\top x\right]_+p_S(x)\right]$, it is enough to bound the Fourier coefficient of $\frac{\text{Maj}_{S\backslash\{1\}}+1}{2}$ with respect to $S\backslash\{1\}$. Fortunately, a closed form expression is provided in \citep[Theorem 5.19]{odonnell2021analysisbooleanfunctions}. Specifically, since it is assumed that $|S|-1$ is odd, we have:
\begin{align}
    \widehat{D_1g}(S\backslash\{1\})&=\reallywidehat{\frac{\text{Maj}_{S\backslash\{1\}}+1}{2}}(S\backslash\{1\})=\frac{1}{2}\left(\widehat{\text{Maj}_{S\backslash\{1\}}}(S\backslash\{1\})+\widehat{1}(S\backslash\{1\})\right)\nonumber\\
    \label{eq:Fourier_derivative_coef}&=\frac{1}{2}\left(\left(-1\right)^{\frac{|S|-2}{2}}\frac{\binom{\frac{|S|-2}{2}}{\frac{|S|-2}{2}}}{\binom{{|S|-2}}{|S|-2}}\cdot\frac{2}{2^{|S|-1}}\binom{|S|-2}{\frac{|S|-2}{2}}+0\right)=\frac{\left(-1\right)^{\frac{|S|-2}{2}}}{2^{|S|-1}}\binom{|S|-2}{\frac{|S|-2}{2}}~,
\end{align}
where $\widehat{1}(S\backslash\{1\})$ refers to the Fourier coefficient of the constant function 1 with respect to $(S\backslash\{1\})$, which equals $0$. According to Stirling's approximation, it holds for all $n\geq1$ that $\sqrt{2\pi n}\left(\frac{n}{e}\right)^n\exp(\frac{1}{12n}-\frac{1}{360n^3})< n!<\sqrt{2\pi n}\left(\frac{n}{e}\right)^n\exp(\frac{1}{12n})$. Therefore we can bound $\binom{n}{\frac{n}{2}}$ with the following expression:
\begin{align*}
    \binom{n}{\frac{n}{2}}=\frac{n!}{(\frac{n}{2}!)^2}>\frac{\sqrt{2\pi n}\left(\frac{n}{e}\right)^ne^{\frac{1}{12n}-\frac{1}{360n^3}}}{\left(\sqrt{\pi n}\left(\frac{n}{2e}\right)^{\frac{n}{2}}e^{\frac{1}{6n}}\right)^2}=\frac{2^n\sqrt{2}e^{\frac{1}{12n}-\frac{1}{360n^3}}}{\sqrt{\pi n}e^{\frac{1}{3n}}}>\frac{2^n}{2\sqrt{n}}~.
\end{align*}
Substituting $n$ with $|S|-2$ and and plugging the expression above to \eqref{eq:Fourier_derivative_coef} we get
\begin{align*}
    (-1)^{-\frac{|S|-2}{2}}\E\left[[\mathbf{1}_S^\top x]_+p_S(x)\right]&=\widehat{D_1g}(S\backslash\{1\})=(-1)^{-\frac{|S|-2}{2}}\cdot\frac{\left(-1\right)^{\frac{|S|-2}{2}}}{2^{|S|-1}}\binom{|S|-2}{\frac{|S|-2}{2}}\\
    &\geq\frac{1}{2^{|S|-1}}\cdot\frac{2^{|S|-2}}{2\sqrt{|S|-2}}\geq\frac{1}{4\sqrt{|S|}}~.
\end{align*}
Finally, plugging the expression above into \eqref{eq:single_neuron_loss}, we get:
\begin{align*}
    F_S(w,b)&=\frac{1}{8|S|^2}-\frac{(-1)^{\frac{|S|-2}{2}}}{|S|^{\frac{3}{2}}}\E\left[\left[\mathbf{1}_S^\top x\right]_+p_S(x)\right]+1\\
    &\leq\frac{1}{8|S|^2}-\frac{1}{|S|^{\frac{3}{2}}}\cdot\frac{1}{4\sqrt{|S|}}+1=1-\frac{1}{8|S|^2}~.
\end{align*}

\subsection{Proof of \thmref{thm:single}}

The proof is overall highly similar to that of \thmref{thm:hard}, with some small modifications (due to the different loss function and network architecture). We let $\theta_t=(w_t,b_t)$.

Again, we will use the notation $\varepsilon$ as shorthand for $\exp(-|S|/18)$. For simplicity, instead of $(w_0,b_0)\sim\mathcal{N}(0,\sigma^2I_{d+1})$ and $x\sim\{\pm1\}^d$ uniformly, we will consider $w_0\sim\mathcal{N}(0,\sigma^2\cdot I_{d+1})$ and $x\sim\{\pm1\}^d\times\{1\}$. Since $w^\top x+b=(w_0^\top, \,b_0)\cdot\binom{x}{1}$ 
 this change of notation is equivalent to plugging $v=\binom{w_0}{b_0}$ and $y=\binom{x}{1}$, and analyzing the inner product of $v^\top y$ instead of $w^\top x+b$. In both cases the distribution of our random variables is identical. In the same manner, for each PGD iterate we will replace the notation $(w_t,b_t)\in\reals^d\times\reals$ and $x\in\{\pm1\}^d$ with $w_t\in\reals^{d+1}$ and $x\in\{\pm1\}^d\times\{1\}$ . This way, for each $t$,
 $w_t$ and $F_S(w_t)$ are only a function of the random variables $w_0$ and $\{\xi_j\}_{j=0}^{t-1}$. We now define
\begin{align}
    \label{def:F_lin}F_{lin}(w)&=\E_x\left[[w^\top x]_+p_S(x)\right]~,\notag\\
    \left[z\right]_\varepsilon&=z\cdot\mathbf{1}_{\{z >\epsilon\}}\\
    \label{def:delta_quad}\Delta_t&= \eta\nabla F_S(w_t)-\xi_t=\eta\frac{\partial}{\partial w_t}(\E_x[w_t^\top x]_+^2)+2\eta\nabla F_{lin}(w_t)-\xi_t\\
    &=\eta\frac{\partial}{\partial w_t}\left(\frac{\norm{w_t}^2}{2}\right)+2\eta\nabla F_{lin}(w_t)-\xi_t=\eta w_t+2\eta\nabla F_{lin}(w_t)-\xi_t~,\nonumber\\
    \label{def:delta_tilde_quad}\tilde{\Delta}_t&=\eta w_t+2\eta\left[\nabla F_{lin} (w_t)\right]_\varepsilon-\xi_t~.
\end{align}
Next, just like we did in the proof of \thmref{thm:hard}, we define the distributions $P$ and $\tilde{P}$ over $\{w_t\}_{t=0}^T$ by setting the initialization $w_0=\xi_{-1}\sim \mathcal{N}(0,\sigma I_{d+1})$ and the update step as $w_{t+1}=w_t-\Delta_t,\quad w_{t+1}=w_t-\tilde{\Delta}_t$ for $P$ and $\tilde{P}$ respectively. Thus, $P$ is the distribution of the iterates produced by the  PGD algorithm, and $\tilde{P}$ is an alternative distribution where we omit the linear part of the gradient whenever it is negligible (in the sense that $\norm{\nabla F_{lin}(w_t)}<\varepsilon$). For simplicity we are going to use the notation 
\begin{align}
    \label{def:v_t}v_t=\sum_{j=0}^t(1-\eta)^{t-j}\xi_{j-1}~,
\end{align}
and decompose our probability similarly to the procedure in the proof of \thmref{thm:hard}:
\begin{align}
    P(F_S(w_T)\leq1-\varepsilon)&\leq TV(P,\tilde{P})+\tilde{P}(F_S(w_T)\leq1-\varepsilon)=\nonumber\\
    &=TV(P,\tilde{P})+\tilde{P}(F_S(w_T)\leq1-\varepsilon \wedge \forall t: \norm{\nabla F_{lin}\left(v_t\right)}< \varepsilon)\nonumber\\
    \label{eq:randomF_small_quad} &\quad \quad \quad+\tilde{P}( F_S(w_T)\leq1-\varepsilon \wedge \exists t:  \norm{\nabla F_{lin}\left(v_t\right)}\geq \varepsilon)
\end{align}
where $TV(P,\tilde{P})=\sup_{A\in\mathcal{F}}
\left|P(A)-\tilde{P}(A)\right|$ is the total variance distance between $P$ and $\tilde{P}$. Dealing with each expression in \eqref{eq:randomF_small_quad} individually, we will start from the second one. We claim that under the distribution $\tilde{P}$, the condition  $\forall t \norm{\nabla F_{lin}(v_t)}< \varepsilon$ implies $\forall t \,\,\, w_t=v_t$. For $t=0$ we have $v_0=(1-\eta)^{0-0}\xi_{-1}=w_0$ so the claim is trivial. Assuming the claim is valid for some $t$, we have 
\begin{align*}
    w_{t+1}&=v_t-\eta v_t-2\eta \left[\nabla F_{lin} (v_t)\right]_\varepsilon+\xi_t\\
    &=(1-\eta)v_t+\xi_t
    =(1-\eta)\sum_{j=0}^t(1-\eta)^{t-j}\xi_j+(1-\eta)^0\xi_t\\
    &=\sum_{j=0}^t(1-\eta)^{t+1-j}\xi_j+(1-\eta)^{t+1-(t+1)}\xi_t=\sum_{j=0}^{t+1}(1-\eta)^{t+1-j}\xi_j=v_{t+1}~.
\end{align*}
Therefore, under the distribution $\tilde{P}$, the event $\left\{\forall t \, \norm{\nabla F_{lin}(v_t)} < \varepsilon\right\}$ implies $\{\forall t \, w_t=v_t\}$. Using that we get:
\begin{align*}
    &\tilde{P}\left(F_S(w_T)\leq1-\varepsilon\wedge\forall t \,\, \norm{
    \nabla F_{lin}(w_t)}<\varepsilon\right)\leq\tilde{P}\left(F_S(w_T)\leq1-\varepsilon\wedge\forall t \,\, w_t=v_t\right)\\ 
    &=\tilde{P}\left(F_S(v_T)\leq1-\varepsilon\wedge\forall t \,\, w_t=v_t\right)\leq \tilde{P}\left(F_S(v_T)\leq1-\varepsilon\right)\\
    &=\tilde{P}\left(\E_x\left[\left([v_T^\top x]_+-p_S(x)\right)^2\right]\leq1-\varepsilon\right)=\tilde{P}\left(\E_x[v_T^\top x]_+^2-2\E_x\left[[v_T^\top x]_+p_S(x)]\right]+1\leq1-\varepsilon\right)\\
    &= \tilde{P}\left(2\E_x\left[[v_T^\top x]_+p_S(x)\right]\geq\varepsilon-\E_x[v_T^\top x]_+^2\right)\leq\tilde{P}\left(2\left|\E_x\left[[v_T^\top x]_+p_S(x)
    \right]\right|\geq\varepsilon\right) ~.
\end{align*}
The quantity $\left|\E_x\left[[v_T^\top x]_+p_S(x)\right]\right|$ is the same as the one appearing in the second bound in \lemref{lem:gradsmall}, where we bounded its expected value. Using Markov's inequality and \lemref{lem:gradsmall} on the equation above, we get:
\begin{align}
   \tilde{P}(2\left|\E_x([v_T^\top x]_+p_S(x)\right|)\geq\varepsilon)&\leq\frac{2\E_{v_T}\left([v_T^\top x]_+p_S(x)\right)}{\varepsilon} \nonumber\\
   \label{eq:v_T_small} &\leq\frac{10d\sigma\sqrt{\sum_{j=0}^T(1-\eta)^{2(T-j)}}e^{-\frac{|S|}{8}}}{\varepsilon}\leq\frac{e^{-\frac{|S|}{9}}}{e^{-\frac{|S|}{18}}}=e^{-\frac{|S|}{18}}=\varepsilon~.
\end{align} 
Note that in the last line we used our assumption that $|S|\geq72\log\left(10d\sigma\sqrt{\sum_{j=0}^T(1-\eta)^{2(T-j)}}\right)$.  Returning to \eqref{eq:randomF_small_quad} and considering the third term in the sum, we upper bound it as follows:
\begin{align}
&\tilde{P}\left(F_S(w_T)\leq 1-\varepsilon \wedge \exists t: \norm{\nabla F_{lin}(v_t)}\geq \varepsilon\right) \leq \tilde{P}\left(\exists t: \norm{ \nabla F_{lin}(v_t)}\geq \varepsilon\right) \nonumber\\    \label{eq:allv_t_small}&\stackrel{(1)}{\leq}\sum_{t=0}^T\tilde{P}\left(\norm{\nabla F_{lin}(v_t)}\geq \varepsilon\right)=\sum_{t=0}^T\tilde{P}\left(\norm{\nabla F_{lin}(v_t)}^2\geq \varepsilon^2\right)\nonumber\\ &\stackrel{(2)}{\leq}\sum_{t=0}^T\frac{\E\norm{\nabla F_{lin}(v_t)}^2}{ \varepsilon^2}\stackrel{(3)}{\leq}\sum_{t=0}^T\varepsilon=(T+1)\varepsilon~,
\end{align}
where in (1) we use the union bound,  in (2) we use Markov's inequality, and in (3) we use \lemref{lem:gradsmall} (note that $v_t$ has a Gaussian distribution, as a sum of Gaussian random variables). 

Bringing back the notation $f_S(w)=\E_x[p_S(x)1_{\{w^\top x>0\}}]$ from \thmref{thm:fbound_with_b}, we bound the gradient of $F_{lin}$ using that lemma:
\begin{align*}
    \E_{v_t}\norm{\nabla F_{lin}(v_t)}^2&\stackrel{(1)}{=}\sum_{i=1}^{d+1}\E_{v_t}\left(\E_x^2[x_i\cdot p_S(x)1_{\{v_t^\top x>0\}}] \right)\\
    &\stackrel{(2)}{=}\begin{cases}
        \sum_{i=1}^{d+1}\E_{v_t}\left(f_{S\backslash\{i\}}^2(v_t) \right) & \{i\}\in S\\
        \sum_{i=1}^{d+1}\E_{v_t}\left(f_{S\cup\{i\}}^2(v_t) \right) & \{i\}\notin S
    \end{cases}~~
    \stackrel{(3)}{\leq}~~\begin{cases}
        \sum_{i=1}^{d+1}8e^{-\frac{|S|-1}{4}} & \{i\}\in S\\
        \sum_{i=1}^{d+1}8e^{-\frac{|S|+1}{4}} & \{i\}\notin S
    \end{cases}\\
    &\leq(d+1)11e^{-\frac{|S|}{4}}\leq e^{-\frac{|S|}{18}}=\varepsilon~,
\end{align*}
where $(1)$ is due to a direct calculation of the gradient, $(2)$ is done by plugging the definition of $f_S$ into the different cases, and $(3)$ is done by applying \thmref{thm:fbound_with_b}
(we note that a rather similar calculation is performed in the proof of \lemref{lem:gradsmall} and \eqref{eq:partial_derivative_w_norm_bound}).  

Turning to the total variation term in \eqref{eq:randomF_small_quad}, we prove a slightly different version of \lemref{lem:TV01} (as $P,\tilde{P}$ are defined a bit differently this time):
\begin{lemma}\label{lem:TV01_quad}
    For $P,\tilde{P}$ as defined in \eqref{def:F_lin}, \eqref{def:delta_quad} and \eqref{def:delta_tilde_quad}, it holds that
    \begin{equation*}
    TV\left(P,\tilde{P}\right)~\leq~ \frac{\varepsilon\eta\sqrt{T}}{\sigma}    
    \end{equation*}
\end{lemma}

\begin{proof}
    Using Pinsker's inequality, we have:
    \begin{align*}
        &TV(P,\tilde{P}) \leq \sqrt{\frac{1}{2} KL(P||\tilde{P})}\\
        &=\sqrt{\frac{1}{2}\sum_{t=1}^T\E_{P\left({\{w_i\}_{i=0}^{t-1}}\right)}KL\left(P({w_t})|\{w_i\}_{i=0}^{t-1}||\tilde{P}({w_t})|\{w_i\}_{i=0}^{t-1}\right)+\frac{1}{2}KL\left(P({w_0})||\tilde{P}({w_0})\right)}~,
    \end{align*}
    where KL is the the Kullback-Leibler divergence, and we used the KL-divergence chain rule for the second transition. Using the definitions of $P$ and $\tilde{P}$ we get that the equation above equals: 
    \begin{align}
        \label{eq:squared_KL}&=\sqrt{\frac{1}{2}\sum_{t=1}^T\E_{P\left(\{w_i\}_{i=0}^{t-1}\right)}KL\left(w_{t-1}+\Delta_{t-1}|\{w_i\}_{i=0}^{t-1}||w_{t-1}+\tilde{\Delta}_{t-1}|\{w_i\}_{i=0}^{t-1}\right)}
    \end{align}
    Under both $P,\tilde{P}$ we have $w_0\sim\mathcal{N}(0,\sigma^2I_{d+1})$ and thus $KL(P({w_0})||\tilde{P}({w_0}))=0$.   
Also, since $w_{t-1}$ and $\eta\nabla F_S(
w_{t-1})$ and $\E_x[w_{t-1}^\top x]^2$ are fixed under the conditioning of $\{w_i\}_{i=0}^{t-1}$, definitions \ref{def:delta_quad} and \ref{def:delta_tilde_quad} of $\Delta_{t-1}$ and $\tilde{\delta}_{t-1}$ assures us that the distributions in the expression above are only a function of  $\xi_{t-1}$. Furthermore,  since $\xi_{t-1}\sim\mathcal{N}(0,\sigma^2I)$,  and $w_{t-1}$ is conditionally fixed, we get that
\begin{align*}
    w_{t-1}+\Delta_{t-1}|\{w_i\}_{i=0}^{t-1}&~~\sim~~\mathcal{N}(w_{t-1}-\eta\nabla\E_x[w_{t-1}^\top x]_+^2+2\eta\nabla F_{lin}(w_{t-1}),\sigma^2I_{d+1})\\
    w_{t-1}+\tilde{\Delta}_{t-1}|\{w_i\}_{i=0}^{t-1}&~~\sim~~\mathcal{N}(w_{t-1}-\eta\nabla\E_x[w_{t-1}^\top x]_+^2+2\eta[\nabla F_{lin}(w_{t-1})]_\varepsilon,\sigma^2I_{d+1})~.
\end{align*}
     The Kullback-Liebler divergence of two Gaussian random variables with means $\mu_1,\mu_2$ and covariance matrix $\sigma^2 I$ equals $\frac{\norm{\mu_1-\mu_2}^2}{2\sigma^2}$. Plugging this into \eqref{eq:squared_KL}, it follows that it equals 
    \begin{align*}
         &=\sqrt{\frac{1}{2}\sum_{t=1}^T\E_{P(\{w_i\}_{i=0}^{t-1})}\left[\frac{\norm{
         2\eta\left(\nabla F_{lin} (w_{t-1})-\left[\nabla F_{lin} (w_{t-1})\right]_\varepsilon\right)}^2}{2\sigma^2}\right]}\leq\sqrt{\frac{1}{2}\sum_{t=1}^T\E_{P({\{w_i\}_{i=0}^{t-1}})}\frac{4\eta^2\varepsilon^2}{2\sigma^2}}\\
         &=\frac{\eta\varepsilon\sqrt{T}}{\sigma}~,
     \end{align*}
which concludes the lemma's proof.
\end{proof}
Inserting the results from \eqref{eq:v_T_small}, \eqref{eq:allv_t_small} and \lemref{lem:TV01_quad} into \eqref{eq:randomF_small_quad} we get our desired theorem:
\begin{equation*}
P(F_S(\theta_T)\leq 1-\varepsilon)~\leq~\frac{\eta\varepsilon\sqrt{T}}{\sigma}+\varepsilon+(T+1)\varepsilon~=~\varepsilon\left(\frac{\eta\sqrt{T}}{\sigma}+T+2\right)~.
\end{equation*}

\section{Technical Lemmas}

\begin{lemma}\label{lem:hemispheres}
For $v,u\in S^{d-1}$ we
have: \begin{align*}
    \frac{Vol\left(H_{v}\cap H_{u}\right)}{Vol\left(S^{d-1}\right)}=\frac{\pi-\arccos\left(u^{\top}v\right)}{2\pi}
\end{align*}
Where $H_v=\{w\in S^{d-1}:v^\top w>0\}$ is the open hemisphere around $v$ in the $d$-dimensional unit sphere $S^{d-1}$.
\end{lemma}

\begin{proof}

We begin by dealing with the simple case of linearly dependent $u$ and $v$. Since $u,v\in S^{d-1}$ they are linearly dependent if and only if $u\in\{v,-v\}$. In case $u=v$, the two hemispheres coincide, so we have
\begin{align*}
    \frac{Vol\left(H_{v}\cap H_{v}\right)}{Vol\left(S^{d-1}\right)}=\frac{Vol\left(H_{v}\right)}{Vol\left(S^{d-1}\right)}=\frac{1}{2}=\frac{\pi-\arccos\left(1\right)}{2\pi}=\frac{\pi-\arccos\left(v^\top v\right)}{2\pi}~.
\end{align*}
In the case of $u=-v$, the hemispheres are disjoint, and we have
\begin{align*}
        \frac{Vol\left(H_{v}\cap H_{-v}\right)}{Vol\left(S^{d-1}\right)}=\frac{Vol\left(\emptyset\right)}{Vol\left(S^{d-1}\right)}=0=\frac{\pi-\arccos\left(-1\right)}{2\pi}=\frac{\pi-\arccos\left(-v^\top v\right)}{2\pi}~.
\end{align*}
With those cases solved, we may move on to the the case of linearly dependent vectors. Using spherical coordinates will be much more natural and convenient in this sort of volume calculations. Recall that any point $u\in S^{d-1}$ can be described with $d-1$ angles $\phi=\{\phi_i\}_{i=1}^{d-1}\in[0,\pi]^{d-2}\times[0,2\pi] $, where the transition from spherical coordinates to Cartesian coordinates is done by the map $g:[0,\pi]^{d-2}\times[0,2\pi]\rightarrow S^{d-1}$  which is defined as
\begin{align*}
    \forall i\leq d-1 \,\,\,\,u_i=(g(\phi))_i&=\cos(\phi_i)\prod_{j=1}^{i-1}\sin(\phi_j)\\
    u_d=(g(\phi))_{d}&=\prod_{j=1}^{d-1}\sin(\phi_j)~.
\end{align*}
Also, when integrating over some function $f:S^{d-1}\rightarrow\reals$ (in our case, the constant function $1$), we can switch between the coordinates systems by using the following transition
\begin{align*}
    \int_{S^{d-1}} f(x)dx=\int_{[0,\pi]^{d-2}\times[0,2\pi]} \prod _{j=1}^{d-2}\sin^{d-1-j}(\phi_i)d\phi_1d\phi_2...d\phi_{d-1}~.
\end{align*}
From now on we will describe $u$ and $v$ with spherical coordinates and will use the notation $u_c,v_c$ to describe their Cartesian coordinates. Without loss of generality (up to a change of basis and assuming $u,v$ are linearly independent
vectors), we can assume that $u=(\frac{\pi}{2},\frac{\pi}{2},...,\frac{\pi}{2})$ and $v=(\frac{\pi}{2},\frac{\pi}{2},...,\frac{\pi}{2},\psi)$ (This is equivalent to assuming $u_c=(0,0,..,1)$ and $v_c=(0,0,...,\cos(\psi),\sin(\psi))$ in Cartesian coordinates). The open hemispheres around $u,v$ are $H_u=\{\phi\in S^{d-1}:\phi_{d-1}\in(0,\pi)\}$ and $H_v=\{\phi\in S^{d-1}:\phi_{d-1}\in(\psi\pm\frac{\pi}{2})\,\,\mod{2\pi}\}$ respectively. Since volumes are not affected by rotations and reflections, we may assume that $\psi\in[\frac{\pi}{2},\frac{3\pi}{2}]$ 
(otherwise, we can create this situation by the reflection of the Cartesian coordinate $x_{d-1}$). Thus, the intersection of the hemispheres corresponds to $H_u\cap H_v=\{\phi\in S^{d-1}:\phi_{d-1}\in[\psi-  \frac{\pi}{2},\pi]\}$ .  \\
We now turn to calculating the relevant volumes. Using the transition to spherical coordinates, the volume of the $d$-dimensional unit sphere can be described as
\begin{align*}
Vol\left(S^{d-1}\right)&= \int_{0}^{2\pi}\left(\int_{[0,\pi]^{d-2}}\prod_{i=1}^{d-2}\sin^{d-1-i}\left(\phi_{i}\right)d\phi_1...d\phi_{d-2}\right)d\phi_{d-1}\\
&=2\pi\int_{[0,\pi]^{d-2}}\prod_{i=1}^{d-2}\sin^{d-1-i}\left(\phi_{i}\right)d\phi_1...d\phi_{d-2}~.    
\end{align*}
As for the intersection of $H_u$ and $H_v$, its volume is
\begin{align*}
    Vol\left(H_{u}\cap H_{v}\right)&=\int_{\psi-\frac{\pi}{2}}^{ \pi}\left(\int_{[0,\pi]^{d-2}}\prod_{i=1}^{d-2}\sin^{d-1-i}\left(\phi_{i}\right)d\phi_1...d\phi_{d-2}\right)d\phi_{d-1}\\
&=\left(\frac{3\pi}{2}-\psi\right)\int_{[0,\pi]^{d-2}}\prod_{i=1}^{d-2}\sin^{d-1-i}\left(\phi_{i}\right)d\phi_1...d\phi_{d-2}~.
\end{align*}
Note that the inner product between $u,v$ is exactly $u^\top v=\sin\psi=\cos(\psi-\frac{\pi}{2})$ and hence $\arccos(u^\top v)=\psi-\frac{\pi}{2}$. With this in mind, and combining the  expressions for $Vol(S^{d-1})$ and $Vol(H_u\cap H_v)$ above, we get
\begin{align*}
\frac{Vol\left(H_{u}\cap H_{v}\right)}{Vol\left(S^{d-1}\right)}&=\frac{\frac{3\pi}{2}-\psi}{2\pi}\cdot \frac{\int_{[0,\pi]^{d-2}}\prod_{i=1}^{d-2}\sin^{d-1-i}\left(\phi_{i}\right)d\phi_1...d\phi_{d-2}}{\int_{[0,\pi]^{d-2}}\prod_{i=1}^{d-2}\sin^{d-1-i}\left(\phi_{i}\right)d\phi_1...d\phi_{d-2}}\\
&=\frac{\pi-(\psi-\frac{\pi}{2})}{2\pi}=\frac{\pi-\arccos (u^\top v)}{2\pi}~.
\end{align*}
\end{proof}
\end{document}